%% file: acl_latex.tex
\documentclass[11pt]{article}

\usepackage[final]{acl}

\usepackage{times}
\usepackage{latexsym}

\usepackage[T1]{fontenc}
\usepackage[utf8]{inputenc}

\usepackage{microtype}

\usepackage{inconsolata}

\usepackage{graphicx}

\usepackage{amsmath,amsfonts}
\usepackage{amsthm}
\usepackage[capitalize,noabbrev,nameinlink]{cleveref}
\usepackage{algorithm}
\usepackage{algpseudocode} % algorithmicx
\usepackage{caption}       % for \captionof
\usepackage[most]{tcolorbox}
\usepackage{subcaption}
\usepackage{booktabs}
\usepackage{multirow}

\usepackage{enumitem}
\usepackage{xurl}

\theoremstyle{plain}
\newtheorem{theorem}{Theorem}[section]

\newtheorem{corollary}[theorem]{Corollary}
\theoremstyle{definition}

\newtheorem{assumption}[theorem]{Assumption}

\theoremstyle{remark}

\usepackage{xcolor}
\usepackage{pifont}
\newcommand{\cmark}{\textcolor{green!60!black}{\ding{51}}} % ✓
\newcommand{\xmark}{\textcolor{red!70!black}{\ding{55}}}   % ✗

\newtcolorbox{tbox}[2][]{
  colback=white,
  colframe=blue!75!black,
  fonttitle=\bfseries,
  sharp corners,
  enhanced,
  breakable,
  title={#2}, % This replaces \captionof
  #1
}

\include{math_commands}

\title{Threshold Differential Attention for Sink-Free, Ultra-Sparse, and Non-Dispersive Language Modeling}

\author{
  Xingyue Huang\textsuperscript{1}\thanks{Equal contribution.}\thanks{Work done as intern in Snap Inc.},
  Xueying Ding\textsuperscript{2}\footnotemark[1]\footnotemark[2],
  Mingxuan Ju\textsuperscript{3},
  \textbf{Yozen Liu}\textsuperscript{\textbf{3}},
  \textbf{Neil Shah}\textsuperscript{\textbf{3}},
  \textbf{Tong Zhao}\textsuperscript{\textbf{3}}
  \\[0.6ex]
\textsuperscript{1}University of Oxford\;
\textsuperscript{2}Carnegie Mellon University\;
\textsuperscript{3}Snap Inc.
}

\begin{document}
\maketitle
\begin{abstract}

Softmax attention struggles with long contexts due to structural limitations: the strict sum-to-one constraint forces attention sinks on irrelevant tokens, and probability mass disperses as sequence lengths increase. 
We tackle these problems with Threshold Differential Attention (TDA), a sink-free attention mechanism that achieves ultra-sparsity and improved robustness at longer sequence lengths without the computational overhead of projection methods or the performance degradation caused by noise accumulation of standard rectified attention. TDA applies row-wise extreme-value thresholding with a length-dependent gate, retaining only exceedances. Inspired by the differential transformer, TDA also subtracts an inhibitory view to enhance expressivity. 
Theoretically, we prove that TDA controls the expected number of spurious survivors per row to $O(1)$ and that consensus spurious matches across independent views vanish as context grows. Empirically, TDA produces $>99\%$ exact zeros and eliminates attention sinks while maintaining competitive performance on standard and long-context benchmarks. 
\end{abstract}

\section{Introduction}
The widespread success of large language models (LLMs)~\citep{vaswani2017attention,achiam2023gpt,touvron2023llama} has fueled growing interest in architectural innovations that can scale to increasingly long contexts~\citep{beltagy2020longformer,zaheer2020big,dao2022flashattention,xiao2023streamingllm}. 
In most modern architectures, Softmax serves as a core activation function: converting scaled dot products into normalized attention weights. 
During pretraining, Softmax is favored for its differentiable normalization with efficient vectorizable implementations
\citep{vaswani2017attention,joulin2017efficient}. 
However, despite its ubiquity, there has been a growing debate over its suitability for next-generation LLMs.
Specifically, Softmax's strict sum-to-one constraint induces two major pathologies: (\emph{i}) \emph{attention sink}, where the model is forced to assign non-trivial probability mass to irrelevant tokens to satisfy normalization constraints; and (\emph{ii}) \emph{attention dispersion}, where probability mass is progressively diluted as sequence length grows, weakening the model's focus on salient tokens
and making it inefficient for long-sequence modeling.

To mitigate the aforementioned limitations, prior work has explored \emph{sparse attention mechanisms} that assign exact zero weights to alleviate attention dispersion~\citep{velickovic2024softmax}. Various methods \citep{martins2016sparsemax,peters2019sparse,asentmax2025} replace exponential normalization with sparse projections to improve long-context behavior. However, these approaches typically rely on computationally expensive sorting or iterative projection that scale poorly.
Moreover, as they still enforce sum-to-one normalization, they do not fundamentally resolve the attention sink phenomenon.

Alternatively, \emph{rectified attention mechanisms} replace softmax with ReLU-style activations (e.g., ReLA)~\citep{zhang2021rela,shen2023reluformer}, removing the sum-to-one constraint and naturally inducing sparsity without expensive sorting. 
While early proposals such as ReLA is efficient without enforcing attention sink, they often underperform dense Softmax attention, particularly in long-context regimes, due to limited expressivity and noise accumulation.

In this work, we revisit rectified attention and identify two core challenges to close the performance gap: increasing performance via enhancing expressive capacity, and mitigating long-context degradation by suppressing extreme-value noise. 
To enhance expressivity, we adopt the underlying view of differential Transformer by computing an \emph{inhibitory} view to inhibit the noise, which enables negative attention weights, thereby increasing expressive capacity. 
To tackle long-context limitations, we derived from extreme-value theory and introduced a thresholding mechanism that progressively filters noise as the context length grows.
We term this approach Threshold Differential Attention (\textbf{TDA}). Our contribution includes: 

\begin{itemize}[leftmargin=0.4cm,nosep]
    \item We propose TDA, a drop-in, \emph{non-softmax} attention mechanism. By applying length-dependent thresholding with an inhibitory differential view, TDA yields a sink-free, ultra-sparse attention that is robust to long-context modeling.
    \item We provide a theoretical analysis under sub-Gaussian assumptions. While spurious attention scores in dense models grow with context length, we prove that TDA effectively controls this noise: the expected number of spurious survivors per row remains $O(1)$, and consensus spurious matches across independent views vanish.
    \item Empirically, we show that TDA achieves $>99\%$ sparsity while maintaining competitive performance on both standard and long-context QA benchmarks, and eliminates attention sinks across layers and heads. 
    \item We provide a fused Triton kernel implementation\footnote{\url{https://github.com/snap-research/TDA.git}}
    and show that the resulting TRA kernel is competitive with FlashAttention-2 under BF16, achieving consistent speedups at long contexts. 
\end{itemize}

\section{Related Works} 

\begin{table}[t]
\centering
% \small
\setlength{\tabcolsep}{4pt}
\resizebox{0.48\textwidth}{!}{
\begin{tabular}{lcccc}
\toprule
\textbf{Method} &
\textbf{Exact 0} &
\textbf{Negative} &
\textbf{No sum-to-1} &
\textbf{Length-aware} \\
\midrule
Softmax        & \xmark & \xmark & \xmark & \xmark \\
SSMax          & \xmark & \xmark & \xmark & \cmark \\
Entmax         & \cmark & \xmark & \xmark & \xmark \\
LSSA          & \xmark & \xmark & \xmark & \cmark \\
ReLA           & \cmark & \xmark & \cmark & \xmark \\
Diff Softmax   & \xmark & \cmark & \cmark & \xmark \\
\midrule
TRA (ours)     & \cmark & \xmark & \cmark & \cmark \\
TDA (ours)     & \cmark & \cmark & \cmark & \cmark \\
\bottomrule
\end{tabular}
}
\caption{\textbf{Feature comparison of attention mechanisms.}
\textbf{Exact 0}: can output exact zeros.
\textbf{Negative}: can output signed weights.
\textbf{No sum-to-1}: does not enforce row normalization.
\textbf{Length-aware}: explicitly depends on sequence length.}
\label{tab:attn_feature_comparison}
\end{table}

\paragraph{Attention Sink.} The attention sink phenomenon, wherein models allocate excessive probability mass to initial tokens, is a structural byproduct of the softmax sum-to-one requirement. \citet{barbero2025firsttoken} provide a formal analysis connecting attention sinks to the prevention of representational collapse in deep Transformers. Similarly, \citet{gu2025emerges} show that sinks emerge as learned key biases that store excess mass without influencing value computation, and that replacing softmax with sigmoid-style attention can prevent their emergence. To mitigate this, Softpick~\citep{zuhri2025softpick} uses a non-sum-to-one kernel to bypass sinks by design. Recently, \citet{qiu2025gatedattention} proposed \emph{Gated Attention}, which applies a head-specific sigmoid gate after SDPA; they report that the induced sparsity mitigates attention sinks (and ``massive activations'') and improves long-context extrapolation.

\paragraph{Sparse Attention for Attention Dispersion.} 
Standard attention mechanisms compute a dense probability distribution that tends toward uniformity as sequence length increases, an effect known as \emph{attention dispersion}~\citep{velickovic2024softmax}. To counter this, several methods were proposed to produce exact-zero attention weights. \citet{martins2016sparsemax} and the $\alpha$-entmax family \citep{peters2019sparse} replace the softmax exponential with a transformation that projects logits onto a sparse support. 
Recent variants like ASEntmax \citep{asentmax2025} and AdaSplash~\citep{goncalves2025adasplash} further refine this by adaptively calibrating sparsity or utilizing GPU-efficient kernels to improve length generalization.

Orthogonal to projection methods, rectified attention models such as ReLA \citep{zhang2021rela} show high performance without the normalization constraints of the exponential kernel. ReLUFormer \citep{shen2023reluformer} replaces Softmax with ReLU-style activations inside attention and showcases improved performance on long-sequence settings. 
\citet{zuhri2025softpick} advance this by removing the sum-to-one constraint, employing a non-probabilistic, rectified activation that yields sparse outputs without the computational overhead of iterative projection. Furthermore, Sliced ReLU Attention \citep{boufadene2025sliced} demonstrates that applying ReLU to projected scores maintains the expressive power of standard Transformers while enabling quasi-linear efficiency via sorting. 

\paragraph{Differential Attention.}
Differential Attention \citep{ye2025diff} introduces an inhibitory view that enables explicit noise cancellation, yielding a \emph{signed}
attention signal, encoding both positive and negative interactions.
Dex~\citep{kong2025dex} further analyzes the interaction between this machanism and pretrained self-attention, proposing modifications that retain pretrained capabilities while unlocking the benefits of differential inhibition.

\paragraph{Attention Scaling for Length Generalization.}
Complementary to changing the attention activation, long-context performance can be improved via length-dependent scaling.
\citet{nakanishi2025scalable} propose SSMax to adjust attention scaling with sequence length to mitigate softmax flattening ,
and recent theory formalizes a critical scaling regime for long-context transformers \citep{chen2025criticalattention}.
Relatedly, LSSA~\citep{gao2025softplus} replaces the exponential nonlinearity with a Softplus transform and introduces a length-dependent scaling factor
for better extrapolations.  
We summarize representative methods in \cref{tab:attn_feature_comparison}.

\section{Preliminaries}
\label{sec:prelim}

\paragraph{Notation.}
Let $T\in\mathbb{N}$ denote the sequence length and $\gV$ be the vocabulary with size $|\gV|$.
A token sequence is $\vx_{1:T}$, where $\vx_i\in\{1,\dots,|\gV|\}$.
The model (embedding) dimension is $d_{\mathrm{model}}$, the number of layers is $L$, the number of attention heads is $H$,
and the per-head dimension is $d$.
$\vzero$ denotes an all-zeros matrix of the appropriate shape.
We write $\langle \cdot,\cdot\rangle$ for the Euclidean inner product.

\paragraph{Transformer architecture.}
A Transformer stacks $L$ identical blocks, each consisting of multi-head self-attention (MHSA) and a position-wise
feed-forward network (FFN), wrapped with residual connections and normalization
\citep{vaswani2017attention,he2016residual,ba2016layernorm}.
The input to the first layer consists of token embeddings summed with positional encodings. While earlier architectures used absolute position vectors \citep{vaswani2017attention,devlin2019bert}, modern LLMs typically incorporate relative or rotary position information directly into the attention mechanism \citep{shaw2018selfattention,su2021roformer}.
Each layer $\ell\in\{1,\dots,L\}$ computes
\[
\begin{aligned}
\widehat{\mH}^{(\ell)} &= \mathrm{Norm}\Big(\mH^{(\ell-1)} + \mathrm{MHSA}\big(\mH^{(\ell-1)}\big)\Big),\\
\mH^{(\ell)} &= \mathrm{Norm}\Big(\widehat{\mH}^{(\ell)} + \mathrm{FFN}\big(\widehat{\mH}^{(\ell)}\big)\Big),
\end{aligned}
\]
where causal masking (decoder-only LMs) restricts each position to attend to its prefix. 
$\mathrm{Norm}$ denotes LayerNorm \citep{ba2016layernorm}. $\mathrm{MHSA}$ denotes Multi-head self-attention, and $\mathrm{FFN}$ is a Multi-layer Perceptron.

\paragraph{Attention Mechanisms.}
\label{sec:attn}
Let $\vx_{1:T}$ be a length-$T$ token sequence with hidden states $\mX\in\R^{T\times d_{\mathrm{model}}}$.
For a single attention head with head dimension $d$, define projections
\[
\begin{aligned}
\mQ = \mX \mW_Q,\qquad
\mK = \mX \mW_K,\qquad
\mV = \mX \mW_V,
\end{aligned}
\]
where $\mQ,\mK,\mV\in\R^{T\times d}$ and $\mW_Q,\mW_K,\mW_V\in\R^{d_{\mathrm{model}}\times d}$.
Scores are
\[
\begin{aligned}
\mS \;=\; \frac{\mQ\mK^\top}{\sqrt{d}} \in \R^{T\times T},
\qquad
\mS_{ij} \;=\; \frac{\langle \vq_i,\vk_j\rangle}{\sqrt{d}}.
\end{aligned}
\]
With causal masking, query $i$ attends to $\gS(i)=\{1,\dots,i\}$ by setting $\mS_{ij}=-\infty$ for $j\notin\gS(i)$.

\paragraph{Standard (softmax) Attention.}
Softmax attention uses row-stochastic weights \citep{vaswani2017attention}:
\[
\begin{aligned}
\mA_{ij} \;=\; \frac{\exp(\mS_{ij})}{\sum_{t\in\gS(i)} \exp(\mS_{it})},
\qquad
\mO=\mA\mV.
\end{aligned}
\]

\paragraph{Rectified Linear Attention (ReLA).}
ReLA replaces softmax with a rectifier and removes the sum-to-one constraint \citep{zhang2021rela}:
\[
\begin{aligned}
\mA \;=\; \max(\mS,\vzero), \qquad
\mO \;=\; \mathrm{Norm}(\mA\mV).
\end{aligned}
\]
where $\mathrm{Norm}$ is LayerNorm~\citep{ba2016layernorm}.

\paragraph{Differential (softmax) Attention.}
Differential attention constructs two softmax maps ($t_1,t_2$), and subtracts them \citep{ye2025diff,kong2025dex}:
\[
\begin{aligned}
\mA^{(t)} &= \softmax\!\left(\frac{\mQ^{(t)}(\mK^{(t)})^\top}{\sqrt{d}}\right),\quad t\in\{1,2\},\\
\mA^{\Delta} &= \mA^{(1)} - \lambda\,\mA^{(2)},\qquad
\mO = \mA^{\Delta}\mV,
\end{aligned}
\]
where $\lambda$ is a learnable (often layer-dependent) scalar. Note $\mA^{\Delta}$ may be signed. 
% \tz{$\mW^{\Delta}$ in the eq should be $\mA^{\Delta}$?}
% \draft{Need to unify the notation here.}

\paragraph{Attention Dispersion.}
We adapt the notion of \emph{attention dispersion} from \citet{asentmax2025} for unnormalized and signed attention weights. Let $\va_i \in \mathbb{R}^i$ be the vector of attention weights for a query at position $i$ (where $\va_{ij}=0$ if masked). We define the \emph{effective entropy} $H(\va_i)$ as the Shannon entropy of the $\ell_1$-normalized absolute weights:
\[
\hat{p}_{ij} = \frac{|\va_{ij}|}{\|\va_i\|_1 + \epsilon}, \qquad H(\va_i) = - \sum_{j=1}^{i} \hat{p}_{ij} \log \hat{p}_{ij}.
\]
The attention mechanism is said to be \emph{dispersive} if the effective entropy grows at the same rate as the maximum possible entropy (i.e., the uniform distribution):
\[
\lim_{i \to \infty} \frac{\mathbb{E}[H(\va_i)]}{\log i} = 1.
\]

Conversely, the mechanism is \emph{non-dispersive} if this ratio approaches 0.

\paragraph{Attention Sink.}
Attention sinks describe abnormally large attention allocated to a fixed \emph{position} (often the first token)~\citep{xiao2023streamingllm,gu2025emerges}.
Importantly, some attention mechanisms we consider produce weights that are not probabilities (e.g., ReLA is unnormalized, differential attention can be signed). To compare positional dominance across such mechanisms, for layer $\ell$ and head $h$, we define a
\newterm{generalized sink ratio} by first $\ell_1$-normalizing absolute weights per query following \citet{gu2025emerges}:
\begin{equation*}
\tilde{\mathbf{A}}^{\ell,h}_k = \frac{1}{T-k+1} \sum_{i=k}^{T} \frac{\big|\mathbf{A}^{\ell,h}_{i,k}\big|}{\sum_{t\in\mathcal{S}(i)} \big|\mathbf{A}^{\ell,h}_{i,t}\big|}
\end{equation*}
and then reporting the times-uniform ratio
$\mathrm{gSinkRatio}^{\ell,h}(k):=\tilde{\mathbf{A}}^{\ell,h}_k/\tilde{\mathbf{A}}^{\mathrm{unif}}_k$, where
$$\tilde{\mathbf{A}}^{\mathrm{unif}}_k:=\frac{1}{T-k+1}\sum_{i=k}^{T}\frac{1}{|\gS(i)|}.$$
We write $\mathrm{gSinkRatio}(k)$ as the average (over layers and heads) of the ratio between the total attention mass assigned to key position $k$.

\section{Methodology}
We first introduce Threshold Rectified Attention (TRA), which scales the rectification threshold with context length to suppress extreme-value noise. We then extend it to Threshold Differential Attention (TDA), which subtracts an inhibitory thresholded view to further cancel spurious matches noise.

\subsection{Threshold Rectified Attention (TRA)}
\label{sec:tra}

Rectified attention replaces Softmax with a simple rectifier, producing \emph{un-normalized} and often \emph{sparse} weights: it assigns exact zeros and avoids the sum-to-one constraint that underlies attention sinks
\citep{zhang2021rela,shen2023reluformer}. 
However, plain ReLA-style attention often underperforms Softmax attention in long-context regime~\citep{zuhri2025softpick}.  
% Our working hypothesis is that the failure mode is 
We attribute this to \emph{noise accumulation}: as the context grows, unrelated
(query, key) pairs produce larger \emph{maximum} dot-products by chance (extreme values), and a \emph{fixed} rectifier
threshold eventually fails to inhibit such spurious pairs, polluting the value aggregation.

\begin{assumption}[Sub-Gaussian noise per row]
\label{ass:tra-subg}
Fix a query position $i$ 
% and consider the 
with visible set $\gS(i)=\{1,\dots,i\}$.
For any \emph{noise} key $j \in \gN(i)$, the similarity
$\vs_{ij}=\langle \tilde \vq_i,\tilde \vk_j\rangle$ is mean-zero and
$\sigma^2/d$-sub-Gaussian: 
\(
\mathbb{E}\!\left[\exp(t \vs_{ij})\right] \le \exp\!\Big(\frac{\sigma^2}{2d}t^2\Big),
\forall t\in\mathbb{R}.
\)
\end{assumption}
\begin{assumption}[Bounded relevant survivors]
\label{ass:bounded-relevant}
For each row, the number of relevant keys exceeding threshold is uniformly bounded by a constant $r$.
\end{assumption}

\paragraph{Definition.} Let $\vq_i,\vk_j,\vv_j\in\mathbb{R}^{d}$ denote the per-head query, key, and value vectors.
We normalize queries and keys,
\(
\tilde \vq_i := \frac{\vq_i}{\|\vq_i\|_2} \text{ and } \tilde \vk_j := \frac{\vk_j}{\|\vk_j\|_2},
\)
% Define the (scaled) score
and compute scores as
$
\vs_{ij} := \langle \tilde \vq_i,\tilde \vk_j\rangle.
$
TRA applies a \emph{length-dependent}
threshold $\tau_i$:
\begin{align}
\tau_i &:=  \beta \sqrt{\frac{2\log\!\big(\frac{i+1}{\kappa}\big)}{d}}, \qquad \kappa>0. \label{eq:tau}\\
\va_{ij} &:= \big(\vs_{ij}-\tau_i\big)_+^{\,p}\,\quad \text{if $j\le i$}, \nonumber\\
\vo_i &:= \mathrm{Norm}\!\left(\sum_{j=1}^{i} \va_{ij} \vv_j\right). \nonumber
\end{align}

Here $(x)_+ = \max(x,0)$, $\beta>0$ is a learnable scalar (that controls the threshold size and overall sparsity), $p\ge 1$ is the power, and $\mathrm{Norm}$ is RMSNorm. 
The \emph{row-wise gate} $\tau_i$ increases with
causal context size $|\gS(i)|=i$, ensuring the rectifier remains selective as $i$ grows. 
This scaling follows the extreme-value behavior of sub-Gaussian noise: \citet{vershynin2018high} gives the bound $\mathbb{P}(\max_{j\le i} X_j > \tau)\le i\exp(-\tau^2/(2\sigma^2))$; applying this to dot-products with variance proxy $\mathcal{O}(1/d)$ yields the threshold scale $\tau \asymp \sqrt{2\log i/d}$.

\paragraph{Properties.} Let $\gN(i)\subseteq \gS(i)$ denote the set of noise keys (irrelevant to the query) at row $i$. A key is a \emph{spurious survivor} if $j\le i$ and $j\in\gN(i)$ yet $\vs_{ij}>\tau_i$.
The total number of spurious survivors is
\[
S_i\;:=\; \sum_{j\in\gN(i)} \1\!\left(\vs_{ij}>\tau_i\right),
\]
where $\1$ is the indicator function.
$\tau_i$ tracks the extreme-value scale of spurious dot-products, and hyperparameter $\kappa$ to control the \emph{expected number of spurious survivors per row}.
Hence, TRA ensures the tail exceedance remain controlled, yielding stable sparsity and reducing long-context corruption from chance matches.

\begin{theorem}[TRA keeps $O(1)$ spurious survivors per row]
\label{thm:tra-survivors}
Under Assumption~\ref{ass:tra-subg}, fix any $\kappa>0$, then for all $i\ge 1$,
\(
\mathbb{E}[S_i] \;\le\; \kappa.
\)
\end{theorem}
Under the sub-Gaussian noise model, 
a length-dependent threshold with $\beta \ge \sigma$ keeps the expected number of spurious survivors per query row bounded as context length grows (Proof in \cref{ssec:theorem43_proof}) and prevents noise accumulation from dominating long-context aggregation. Hence, TRA is non-dispersive (Proof in \cref{ssec:theorem44_proof}).

\begin{figure*}[t]
    \centering
    \begin{subfigure}[t]{0.3\textwidth}
        \centering
        \includegraphics[width=\linewidth]{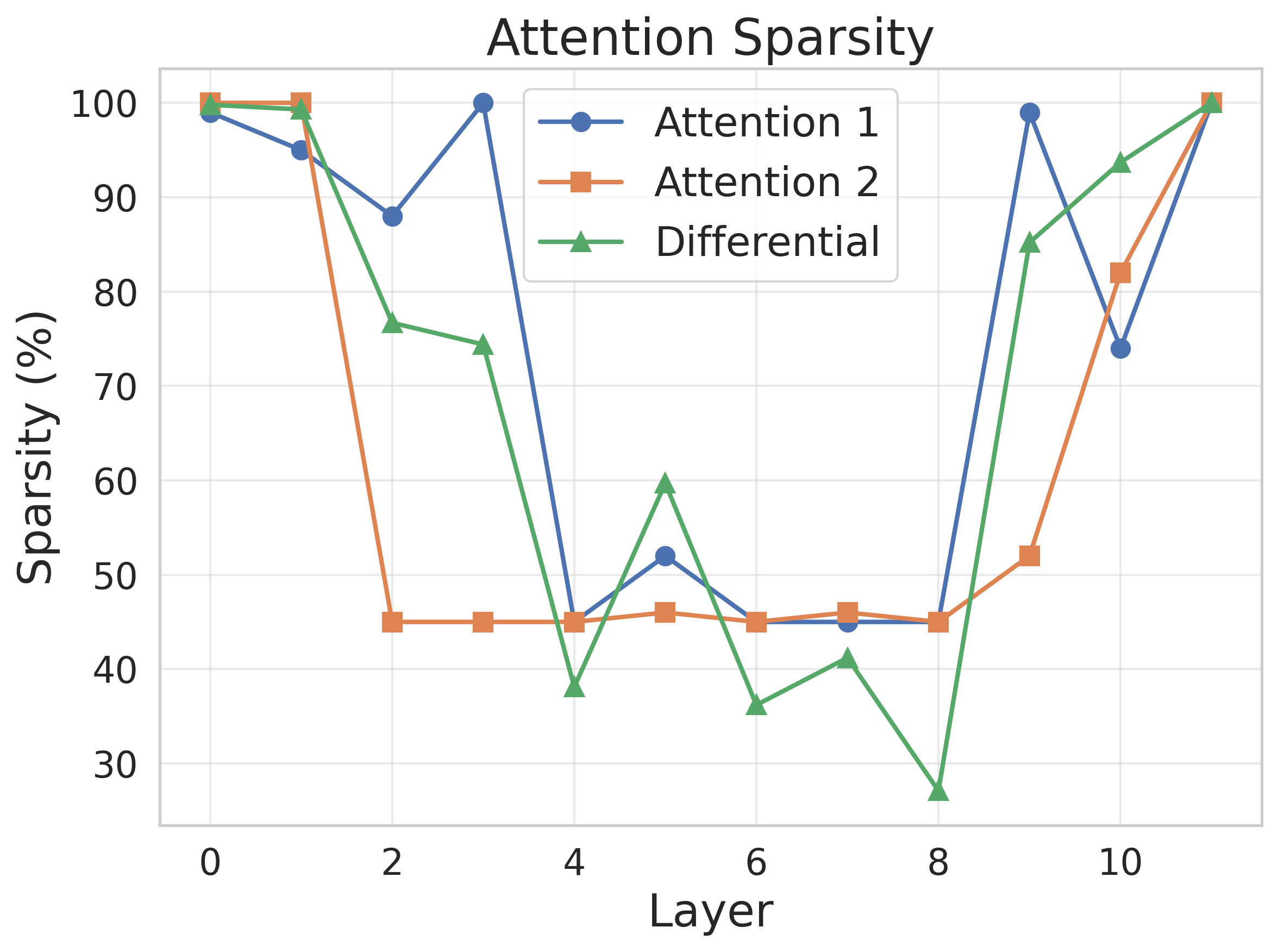}
        \caption{\textbf{Attention Sparsity.} }
        \label{fig:tda_sparsity_curve}
    \end{subfigure}
    \hfill
    \begin{subfigure}[t]{0.3\textwidth}
        \centering
        \includegraphics[width=\linewidth]{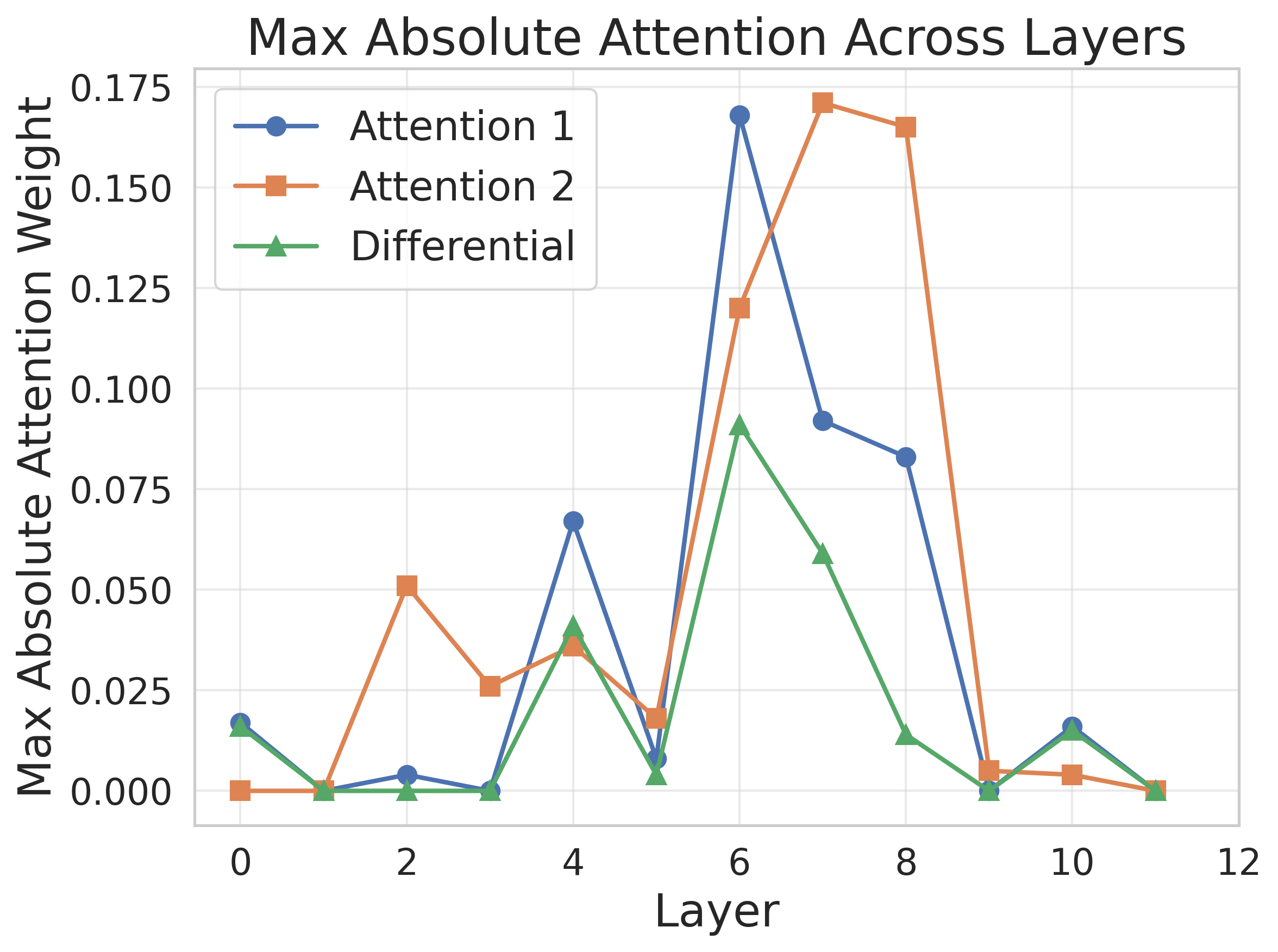}
        \caption{\textbf{Max Absolute Attention.} }
        \label{fig:tda_max_curve}
    \end{subfigure}
    \hfill
    \begin{subfigure}[t]{0.38\textwidth}
        \centering
        \includegraphics[width=\linewidth]{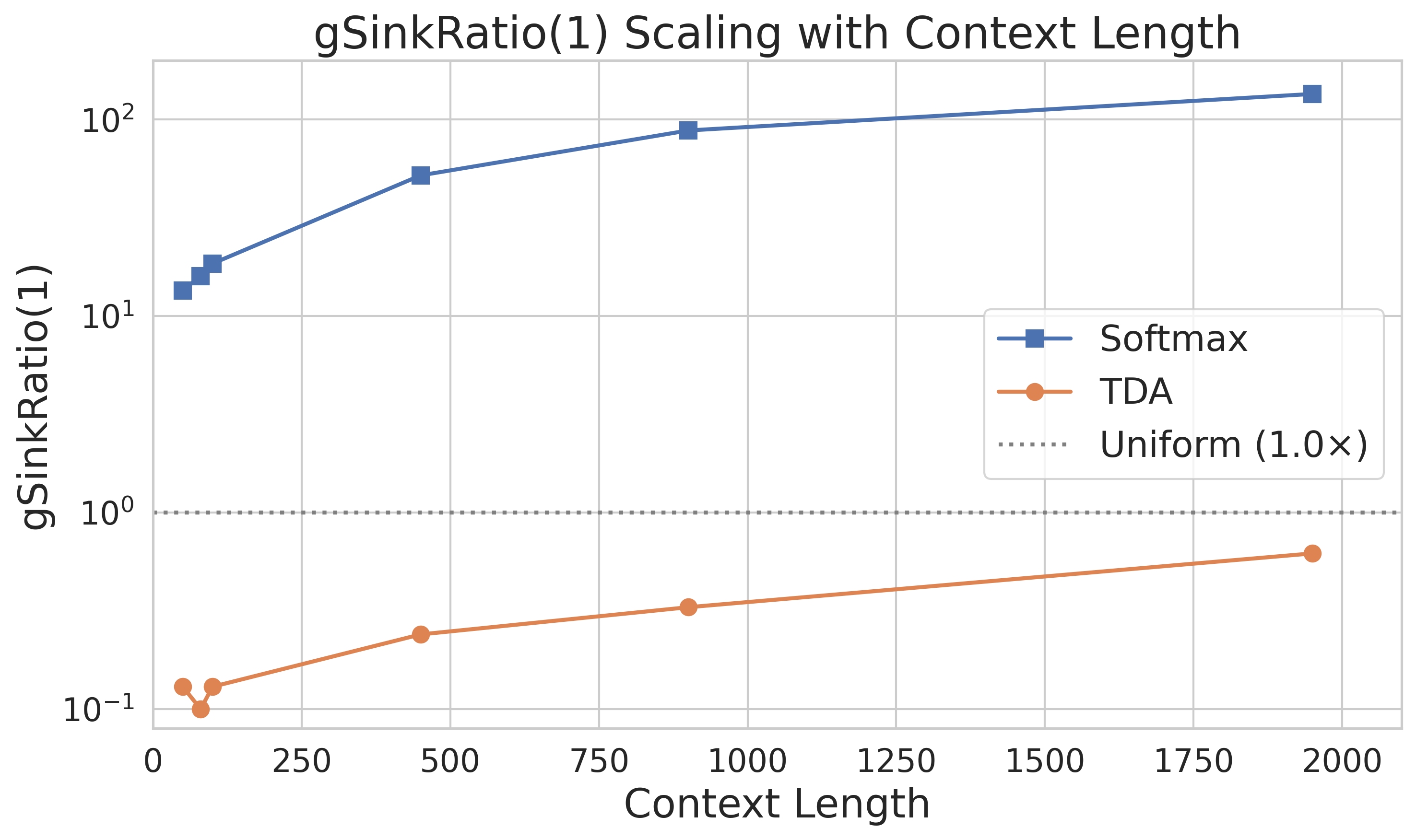}
        \caption{\textbf{Attention Sink Scaling.} }
        \label{fig:tda_sink_scale}
    \end{subfigure}
    
    % \vspace{-0.5em}
    \caption{\textbf{Mechanistic diagnostics for TDA.} We visualize (a) the sparsity of attention weights across layers, (b) the maximum absolute attention values, and (c) the first token attention sink ratio \(\mathrm{gSinkRatio}(1)\) as context length increases.  Attention 1 indicates an excitatory view, while Attention 2 indicates an inhibitory view. }
    \vspace{-0.5em}
    \label{fig:tda_diagnostics}
\end{figure*}

\begin{theorem}[TRA is Non-Dispersive]
\label{thm:tra-non-dispersion}
Under Assumptions~\ref{ass:tra-subg} and \ref{ass:bounded-relevant}, the Threshold Rectified Attention (TRA) is non-dispersive.
\end{theorem}

\subsection{Threshold Differential Attention (TDA)}
\label{sec:tda}

While TRA bounds the \emph{number} of spurious exceedances per row, a single view thresholded map can still admit occasional high-magnitude noise. 
To suppress these, we incorporate the key idea of \emph{differential attention}
\citep{ye2025diff,kong2025dex}: we compute an \emph{excitatory} thresholded view and subtract an \emph{inhibitory} thresholded view. 
Intuitively, a large similarity measure $\vq^\top \vk$ can also occur spuriously due to shared, uninformative structure. In TDA, the inhibitory TRA is trained to capture such non-selective exceedances.

\paragraph{Definition.}
We employ two independent sets of projections $\{\vq^{(t)}, \vk^{(t)}\}_{t\in\{1,2\}}$. We normalize queries and keys in each view and compute scores $\vs_{ij}^{(t)} := \langle \tilde \vq_i^{(t)}, \tilde \vk_j^{(t)}\rangle$. Using the same length-dependent threshold $\tau_i$ as in TRA (\cref{eq:tau}), we compute the differential weights:
\begin{align}
\va_{ij}^{(t)} &:= (\vs_{ij}^{(t)}-\tau_i)_+^{p}, \qquad t\in \{1,2\}, \nonumber\\
\Delta \va_{ij} &:= \va_{ij}^{(1)} - \lambda\va_{ij}^{(2)} \qquad \lambda\in(0,1), \nonumber\\
\vo_i &:= \mathrm{Norm}\left(\sum_{j=1}^{i} \Delta \va_{ij}\vv_j\right).\nonumber
\end{align}
Here $\lambda$ is a learned scalar controling inhibition strength. Similar to \citet{ye2025diff}, $\Delta \va_{ij}$ can be negative, yielding a signed attention signal.

\paragraph{Properties.}
We define \emph{consensus spurious survivors} $C_i$ as the number of noise keys exceeding the threshold in \emph{both} views simultaneously.
\[
C_i := \sum_{j\in\gN(i)} \1\!\left(\vs_{ij}^{(1)}>\tau_i\right)\,\1\!\left(\vs_{ij}^{(2)}>\tau_i\right)
\]

\begin{assumption}[Independence of noise views\footnote{We assume independent noise for tractability; with positive dependence, joint exceedances may increase, and TDA then relies more on differential cancellation than filtration.}]
\label{ass:tda-indep}
Fix a query position $i$. For any noise key $j\in\gN(i)$, the similarities
$\vs_{ij}^{(1)}$ and $\vs_{ij}^{(2)}$ are independent.
\end{assumption}

\begin{theorem}[Consensus spurious survivors vanish]
\label{thm:tda-consensus}
Under Assumption~\ref{ass:tra-subg} and Assumption~\ref{ass:tda-indep} for both views (with the same $\sigma$), for all $i\ge 1$,
\(
\mathbb{E}[C_i] \;\le\; \frac{\kappa^2}{i+1}.
\)
Thus, $\lim_{i\to\infty} \mathbb{E}[C_i] = 0$.
% $\mathbb{E}[C_i]\to 0$ as $i\to\infty$.
\end{theorem}

This formalizes the benefit of TDA: while each view may admit $O(1)$ spurious exceedances, 
the probability of them overlapping on the same noise token vanishes as context length increases (Proof in \Cref{ssec:theorem45_proof}). 
Additionally, as a linear combination of two TRA views, TDA naturally inherits the non-dispersive property (Proof in \Cref{ssec:corollary47_proof}):

\begin{corollary}[TDA is Non-Dispersive]
\label{cor:tda-non-dispersion}
Under Assumptions~\ref{ass:tra-subg}, \ref{ass:bounded-relevant}, and \ref{ass:tda-indep} for both views (with same $\sigma$), then TDA is non-dispersive.
\end{corollary}

We provide additional empirical diagnostics for all the assumptions in \Cref{app:theory-diagnostics}.

\section{Mechanistic Analysis of TDA} 

We begin by mechanistically diagnosing the behavior of TDA to illustrate its theoretical properties. For these analyses, we use a modified GPT-2 architecture~\citep{radford2019language} where the standard absolute positional embeddings are replaced with Rotary Positional Embeddings (RoPE) \citep{su2021roformer} and the Softmax attention is replaced by TDA. 
We use the sentence \emph{``The quick brown fox jumps over the lazy dog''} as input to visualize internal attention dynamics and inhibition patterns.

\paragraph{Attention Sparsity.}

\begin{figure*}[t]
    \centering
    \includegraphics[width=\linewidth]{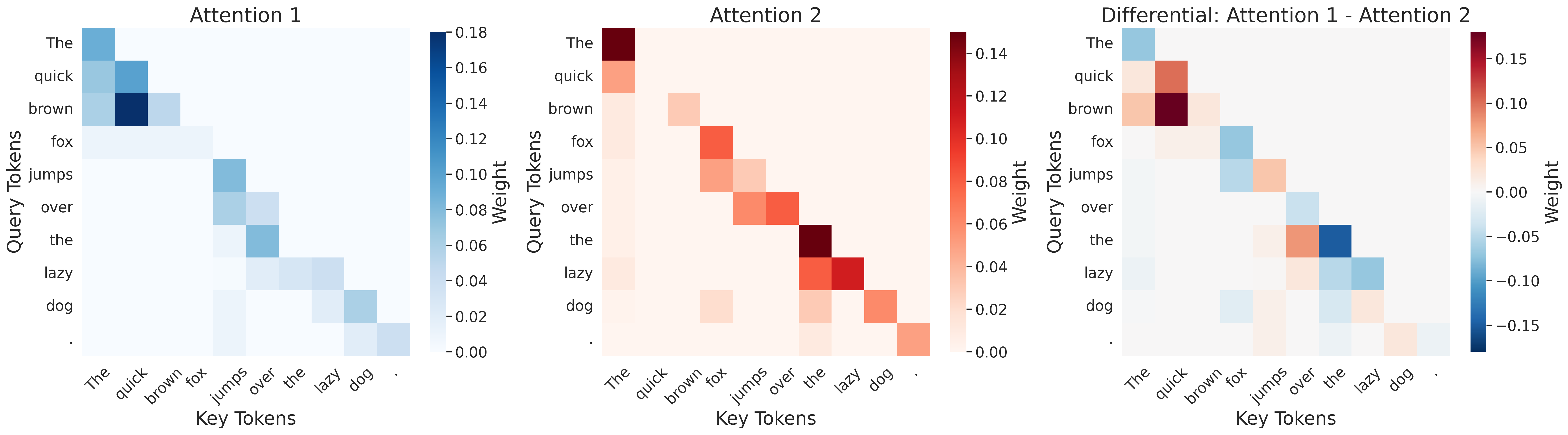}
    \caption{\textbf{Inhibition in TDA.}
    Per-token attention maps for a representative head (Layer 6, Head 0) on the sequence
    \emph{``The quick brown fox jumps over the lazy dog''}.
    The differential attention maps show \(\Delta \mA =  \mA^{(1)} - \lambda  \mA^{(2)}\) where $\lambda=1$; negative values indicate \emph{inhibition}, while positive values indicate \emph{excitation}.}
    \vspace{-1em}
    \label{fig:tda_sparse_heatmap}
\end{figure*}

A defining feature of TDA is that it produces intrinsically sparse attention maps without explicit top-$k$ truncation.
To quantify this, we mark an attention entry as \emph{inactive} if its magnitude is exact 0, and report sparsity as the fraction of inactive entries per layer
(\Cref{fig:tda_sparsity_curve}).
We observe a depth-dependent profile: early and late layers are highly sparse, whereas middle layers are
substantially more active (near-zero rate drops to $\sim$50\%).
This ``active core'' is consistent with representations in intermediate layers, producing stronger and query--key alignments, so a larger fraction of pairs exceed the row-wise threshold and contribute
to value aggregation; outside this region, most interactions are gated off, and further cancel 
differential subtraction, yielding near-all-zero maps.

\paragraph{Attention Sink.}
\Cref{fig:tda_max_curve} shows the layerwise maximum absolute attention weight in each view and in the effective differential map $\Delta\mA=\mA^{(1)}-\lambda\mA^{(2)}$.
Across layers, the differential map exhibits substantially smaller peaks than either individual view, indicating that inhibition cancels large common-mode exceedances and prevents any single interaction from dominating the aggregation. 

This bounded-peak behavior directly translates to robustness against attention sinks. As shown in \Cref{fig:tda_sink_scale}, under Softmax, the first-token sink ratio $\mathrm{gSinkRatio}(1)$ increases sharply as the sequence grows, reflecting the well-known tendency for early positions to become globally attractive “sinks.” In contrast, TDA maintains a sink ratio near (or below) the uniform baseline as context length increases, confirming that differential inhibition effectively prevents the formation of attention sinks.

\paragraph{Inhibition Behavior.}

The differential panel in \Cref{fig:tda_sparse_heatmap} visualizes the signed weights
\(\Delta \va_{ij} := \va_{ij}^{(1)} - \lambda\,\va_{ij}^{(2)}\).
Negative values correspond to \emph{inhibition}: key positions whose contribution is actively subtracted in the value aggregation.
In the example head, the high-frequency preposition \emph{``the''} is broadly inhibited across many queries, consistent with suppressing globally attractive but semantically weak keys.
By contrast, content tokens such as \emph{``quick''} and \emph{``brown''} show \emph{query-dependent} inhibition: they contribute only for the subset of queries where they are useful, and are inhibited elsewhere.
Overall, the inhibition behavior enables fine-grained suppression of redundant keys while preserving selective interactions.

\section{Experiments}
\label{sec:experiment}

We pretrain all models on the FineWebEdu-10B dataset from scratch, a high-quality educational subset of the FineWeb dataset \citep{penedo2024fineweb}. The dataset consists of 10B tokens, and we reserve the first 100M tokens for validation. We train a variant of GPT-2-162M~\citep{radford2019language} model, replacing the learnable positional encoding with RoPE~\citep{su2021roformer}. Throughout the experiments, we set $\kappa=1,\beta=1$, and $p=2$ for TRA and TDA.
When extending models to longer context lengths, we employ NTK-aware scaling for RoPE-based models \citep{peng2024yarn,ntk-aware-scaling}, and further train them on the pretraining dataset for 500 additional steps. We conduct all inference using \texttt{lm-evaluation-harness} library~\citep{eval-harness}.
All experiments were conducted on 8 NVIDIA A100-80GB GPUs. See further experimental details in \cref{app:exp}.

\input{table_results}

\subsection{Language Modeling}

\paragraph{Setup.} For language modelling, we report the zero-shot performance (accuracy and length-normalized accuracy) on a standard suite of multiple-choice commonsense and science QA benchmarks:
HellaSwag~\citep{zellers-etal-2019-hellaswag},
ARC-Easy/Challenge~\citep{clark-etal-2018-arc},
OpenBookQA~\citep{openbookqa},
PIQA~\citep{bisk-etal-2020-piqa},
and Winogrande~\citep{winogrande}.
We compare against (i) \textit{Softmax} attention~\citep{vaswani2017attention} and two length/activation variants: \textit{Gated Softmax}~\citep{qiu2025gatedattention} and \textit{Scalable Softmax (SSMax)}~\citep{nakanishi2025scalable};
(ii) non-softmax baselines including \textit{Entmax}~\citep{martins2016sparsemax,peters2019sparse}, \textit{LSSA}~\citep{gao2025softplus}, and \textit{ReLA}~\citep{zhang2021rela};
and (iii) differential baselines \textit{Differential Softmax}~\citep{ye2025diff} and \textit{Dex}~\citep{kong2025dex}.
Finally, we include \textit{Differential ReLA} as a straightforward combination of ReLA with a differential construction, and our proposed \textit{TRA} and \textit{TDA}.
For a fair comparison, we retrained all baselines and our proposed methods under the same experimental configurations.

\paragraph{Results.} The results on zero-shot common sense reasoning are summarized in \Cref{tab:lmeval_acc_norm_loss_sparsity}. TDA achieves the lowest validation loss and 99\% sparsity while maintaining competitive accuracy across all tasks. Standard Softmax remains a reliable baseline. While variants like Gated and Scalable Softmax aim to improve length robustness, they do not surpass the baseline in the standard context regime.

Among non-softmax methods, unconstrained ReLA suffers from performance degradation due to noise accumulation. Our single-view TRA significantly closes this gap, validating that length-dependent thresholding effectively manages the noise floor. Entmax also shows that sparsity helps with reasoning, achieving competitive accuracy. LSSA performs impressively on ARC-Challenge and Winogrande but remains fully dense, foregoing the efficiency benefits of exact zeros.

Finally, Differential Softmax is the strongest baseline, validating the utility of inhibitory noise cancellation, yet it remains fully dense. Dex attempts efficient correction but trails in accuracy without achieving sparsity, while Differential ReLA gains sparsity at the cost of performance. TDA bridges this gap, matching the high accuracy of Differential Softmax while achieving extreme sparsity by actively filtering noise via thresholding rather than just suppressing it.

\subsection{Long-Context Language Modeling}
\label{sec:long-context}

\begin{table}[t]
\centering
\setlength{\tabcolsep}{4pt}
\resizebox{0.48\textwidth}{!}{
\begin{tabular}{l cccc}
\toprule
\textbf{Method} &
\textbf{QMSum}&
\textbf{SummScreenFD} &
\textbf{GovReport} &
\textbf{Qasper} \\
\midrule

Softmax                & 10.29 &  7.25 & 3.78 &  8.82 \\
SSMax       & 11.22 &  8.47 & 2.68 &  9.70 \\
Diff Softmax   & 10.57 &  8.08 & 3.08 & 11.23 \\
Entmax                 & \textbf{11.52} & \textbf{10.16} & 4.24 & \textbf{11.54} \\
ReLA                   & 11.20 &  9.14 & 4.42 & 10.77 \\
\midrule
TRA ($p{=}2,\beta{=}1$) & 11.18 &  \underline{9.47} & \textbf{5.61} & 11.09 \\
TDA ($p{=}2,\beta{=}1$) & \underline{11.46} &  9.13 & \underline{5.24} & \underline{11.41} \\
\bottomrule
\end{tabular}}
\caption{\textbf{Long-context evaluation on SCROLLS.} We report ROUGE-1 for QMSum, SummScreen, and GovReport, and F1 for Qasper. We \textbf{bold} the first and \underline{underscore} the second place.}

\vspace{-0.35em}
\label{tab:scrolls_long_context}
\end{table}

\paragraph{Setup.}
We evaluate long-context generalization on four tasks from the SCROLLS benchmark~\citep{shaham2022scrolls}:
QMSum, SummScreenFD, GovReport, and Qasper.
We report ROUGE-1 for QMSum, SummScreenFD, and GovReport, and F1 for Qasper, following SCROLLS.
We compare Softmax and representative alternatives: Scalable Softmax, Differential Softmax, Entmax, and ReLA, against our proposed TRA and TDA.

\paragraph{Results.}
\Cref{tab:scrolls_long_context} shows that our methods are consistently competitive: TDA is the second-best method on QMSum and Qasper, while TRA achieves the best GovReport score and ranks second on SummScreenFD.
While Entmax achieves the best overall long-context performance across the four SCROLLS tasks, it is known to be substantially more expensive at long context.
Overall, both TRA and TDA provide strong long-context performance while avoiding projection-based sparse attention overheads.

\subsection{Passkey Retrieval Test}
\label{sec:passkey_retrieval}
\paragraph{Setup.}
We evaluate long-context retrieval using the \emph{passkey retrieval} stress test~\citep{mohtashami2023randomaccess,kamradt2023needle}, in which a short numeric key is inserted at a random position inside a long span of irrelevant text, and the model is later queried to retrieve the key.
We test target context lengths from 500 to 4000 tokens in increments of 500, and run 100 trials with independently sampled passkeys and random insertion locations for each length. See \Cref{app:passkey} for details.

\begin{figure}[t]
    \centering
    \includegraphics[width=\linewidth]{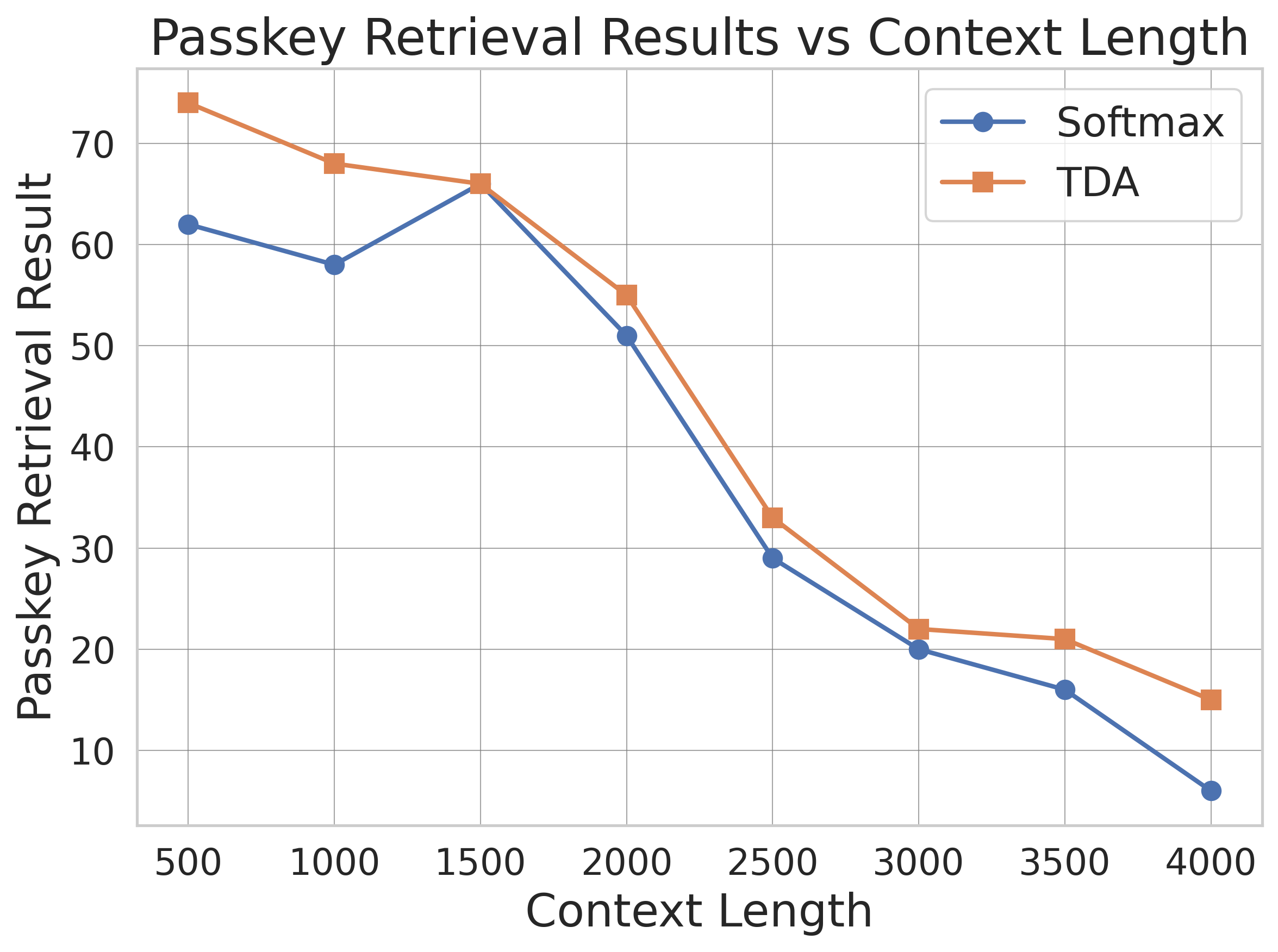}
    \caption{\textbf{Passkey retrieval results.}
    We report correct runs over 100 trials for each context length, with randomly positioned passkey per trial.
    }
    \label{fig:passkey}
\end{figure}

\paragraph{Results.} \Cref{fig:passkey} reports the number of successful trials (out of 100) as a function of context length.
Both Softmax and TDA degrade as the amount of irrelevant text grows, but TDA is consistently more robust across nearly all lengths.
Specifically, TDA consistently yields higher retrieval accuracy than Softmax in shorter contexts, and this advantage persists at long contexts: at 4000 tokens, TDA achieves 15\% correct vs.\ 6\% for Softmax.
Overall, these results suggest that TDA's length-aware thresholding and inhibitory view better suppress chance matches from irrelevant tokens, showing improved retrieval under heavy long-context noise.

\subsection{Multi-needle Retrieval Test} 

\paragraph{Setup.} To further test whether extreme sparsity harms retrieval from multiple distant locations, we also evaluate a \emph{multi-needle} variant in which the model must retrieve several key-value pairs placed at widely separated positions in the context.

\begin{table}[t]
    \centering
    % \small
    \begin{tabular}{@{}lcccc@{}}
        \toprule
        \multirow{2}{*}{\textbf{Context}} & \multicolumn{2}{c}{\textbf{2 Needles}} & \multicolumn{2}{c}{\textbf{4 Needles}} \\
        \cmidrule(lr){2-3} \cmidrule(lr){4-5}
         & \textbf{Softmax} & \textbf{TDA} & \textbf{Softmax} & \textbf{TDA} \\
        \midrule
        500  & 48.0 & \textbf{72.0} & 18.0 & \textbf{62.0} \\
        1000 & 40.0 & \textbf{64.0} & 8.0  & \textbf{28.0} \\
        2000 & 4.0  & \textbf{70.0} & 0.0  & \textbf{2.0}  \\
        4000 & 0.0  & \textbf{82.0} & 0.0  & \textbf{20.0}\\
        \bottomrule
    \end{tabular}
    \caption{\textbf{Multi-needle passkey retrieval.} Avg accuracy of 
    multiple key-value pairs which are inserted at widely separated positions in the context.}
    \label{tab:multi_needle_passkey}
\end{table}

\paragraph{Results.} \Cref{tab:multi_needle_passkey} shows that TDA consistently outperforms Softmax in both the 2-needle and 4-needle settings, with the advantage becoming especially pronounced at longer contexts.
These results suggest that TDA’s sparsity does not eliminate useful global interactions; instead, by filtering irrelevant matches, it preserves the ability to retrieve multiple distant pieces of information under heavy long-context noise.

\subsection{Practical Efficiency}
\label{sec:practical-efficiency}

\begin{table*}[t]
\centering
\begin{tabular}{@{}lcccccccc@{}}
\toprule
\textbf{Context Length} & \textbf{512} & \textbf{1024} & \textbf{2048} & \textbf{4096} & \textbf{8192} & \textbf{16384} & \textbf{32768} & \textbf{65536} \\
\midrule
FlashAttn-2 & 0.108 & \textbf{0.124} & \textbf{0.213} & \textbf{0.568} & 2.009 & 7.781 & 31.713 & 135.616 \\
TRA         & \textbf{0.099} & 0.140 & 0.224 & 0.577 & \textbf{1.773} & \textbf{6.240} & \textbf{24.641} & \textbf{109.803} \\
\midrule
Speedup     & 1.09× & 0.89× & 0.95× & 0.98× & 1.13× & 1.25× & 1.29× & 1.24× \\
\bottomrule
\end{tabular}
\caption{\textbf{BF16 latency (ms) and speedup vs.\ FlashAttention-2.}}
\label{tab:bf16}
\end{table*}

\paragraph{Setup.} A practical advantage of the ultra-sparsity induced by TDA and TRA is the potential reduction in value-aggregation cost. We therefore evaluate the runtime of the proposed TRA kernel under BF16, which is the standard precision for modern LLM training and inference. \Cref{tab:bf16} compares TRA against FlashAttention-2~\citep{dao2023flashattention2} using identical tensor layouts, hardware, and CUDA-event median timing across sequence lengths up to 65k tokens. Additional implementation details and FP32 benchmarks are provided in \Cref{app:Triton_Kernel}.

\paragraph{Results.} TRA is slightly slower at very short sequence lengths (\(\leq\)2k) due to fixed kernel overhead, becomes competitive around 4k, and achieves consistent speedups at longer contexts (\(\geq\)8k), reaching up to \(1.29\times\) at 32k tokens. These gains arise from the sparsity induced by thresholding, which reduces value aggregation cost and becomes more pronounced as sequence length increases.

For completeness, we also report additional FP32 benchmarks comparing our fused Triton implementation against naive PyTorch baselines and fused SDPA operator in \Cref{app:Triton_Kernel}.

\subsection{Hyperparameter Study}

\begin{table}[t]
\centering
\small
\setlength{\tabcolsep}{8pt}

\begin{subtable}[t]{\linewidth}
\centering
\begin{tabular}{c c c}
\toprule
\textbf{$p$} &
\textbf{Val.\ Loss} $\downarrow$ &
\textbf{Avg Accuracy} $\uparrow$ \\
\midrule
1 & 3.2068 & 0.3945 \\
2 & \textbf{3.1190} & \textbf{0.4023} \\
3 & 3.1408 & 0.4020 \\
5 & 3.1412 & 0.3922 \\
\bottomrule
\end{tabular}
\caption{\textbf{Different power $p$} (fix $\beta=1$).}
\label{tab:ablate_k}
\end{subtable}

\begin{subtable}[t]{\linewidth}
\centering
\begin{tabular}{c c c}
\toprule
\textbf{$\beta$} &
\textbf{Val.\ Loss} $\downarrow$ &
\textbf{Avg Accuracy} $\uparrow$ \\
\midrule
1.0 & 3.1190 & \textbf{0.4023} \\
0.8 & \textbf{3.1140} & 0.4015 \\
0.5 & 3.1288 & 0.4018 \\
\bottomrule
\end{tabular}
\caption{\textbf{Different threshold scaling} \textbf{$\beta$}  (fix $p=2$).}
\label{tab:ablate_beta}
\end{subtable}

\caption{\textbf{TDA hyperparameter study.}
Avg Accuracy is averaged over HellaSwag, ARC-Easy, ARC-Challenge, OpenBookQA, PIQA, and Winogrande.}
\label{tab:ablation_tda}
\end{table}

\paragraph{Effect on Power $p$.}
\Cref{tab:ablate_k} studies the power $p$ applied in TRA $(\vs_{ij}-\tau_i)_+^{\,p}$.
We find that a mild nonlinearity is important: $p{=}2$ yields the best validation loss and average accuracy, while $p{=}1$ is noticeably worse, likely because it removes the nonlinearity and thus reduces expressive power, which is consistent with prior observations~\citep{gao2025softplus}.
Larger powers ($p\ge 3$) also slightly degrade performance, likely because aggressively amplifying increases gradient variance

\paragraph{Effect on Threshold Scaling Parameter $\beta$.}
\Cref{tab:ablate_beta} varies the threshold scale $\beta$, which controls the selectivity of the length-aware gate $\tau_i$: smaller $\beta$ admits more noise, whereas larger $\beta$ can over-prune and increase the risk of near-empty rows before the model adapts. Empirically, performance is fairly robust across the tested range, with $\beta{=}1.0$ achieving the best validation loss and the highest average accuracy, while $\beta{=}0.5$ remains close in accuracy but is slightly worse in loss.

\section{Conclusion}
We introduced \emph{Threshold Differential Attention} (TDA), a drop-in non-softmax attention mechanism for 
language modeling that addresses the structural pathologies of attention sinks and dispersion.
TDA combines a length-aware extreme-value threshold with a differential view, yielding signed, ultra-sparse, and sink-free attention.
Theoretically, we prove that spurious survivors per row remains bounded by $O(1)$ and that consensus spurious exceedances across views vanish as context grows. 
Empirically, TDA achieves $>99\%$ exact-zero sparsity while maintaining competitive performance. Furthermore, we provide a fused Triton kernel that translates this into significant runtime and memory gains.
In summary, TDA provides a practical path toward long-context Transformers by improving robustness against dispersion and sinks.

% \clearpage

\section*{Limitations}
We note that due to hardware constraints, the evaluation in this work is primarily at a small scale - scaling TDA to larger models remains an important future work. While we observe consistent behavior across these settings, it remains to be validated whether the same sparsity patterns, training stability, and efficiency gains hold at larger scales (e.g., multi-billion parameter models).

Additionally, overly aggressive thresholding can also cause \emph{dead heads} - heads with no survivors. While head inactivity is not unique to TDA and may in some cases function as an explicit idle state, in the extreme, widespread dead heads can effectively disable multi-head attention in certain layers, reducing the model's expressive capacity for long-range information routing and making performance more sensitive to threshold hyperparameters.
Future work should investigate layer-/head-wise adaptive threshold schedules that preserve ultra-sparsity while preventing head collapse, especially in initial and late layers.

\bibliography{custom}

\appendix
\onecolumn

\section{Proofs}
\label[appendix]{app:proofs}

\begin{assumption}[Sub-Gaussian noise per row (restated Assumption~\ref{ass:tra-subg})]
\label{ass:tra-subg-restated}
Fix a query position $i$ and consider the visible set $\gS(i)=\{1,\dots,i\}$. Let the noise key lies in a subset $\gN \subseteq \gS$.
For any \emph{noise} key $j\in\gN(i)$, the similarity
$\vs_{ij}=\langle \tilde \vq_i,\tilde \vk_j\rangle$ is mean-zero and
$\sigma^2/d$-sub-Gaussian:
\[
\mathbb{E}\!\left[\exp(t \vs_{ij})\right] \le \exp\!\Big(\frac{\sigma^2}{2d}t^2\Big),
\qquad \forall t\in\mathbb{R}.
\]
\end{assumption}

\begin{assumption}[Bounded relevant survivors (restated Assumption~\ref{ass:bounded-relevant})]
\label{ass:bounded-relevant-restated}
Fix a query position $i$ with visible set $\gS(i)=\{1,\dots,i\}$.
Let $\gR(i):=\gS(i)\setminus \gN(i)$ denote the relevant (non-noise) keys.
Then the number of relevant keys that exceed the threshold is uniformly bounded:
\[
R_i := \sum_{j\in\gR(i)} \1(\vs_{ij}>\tau_i) \le r \quad \text{for all } i.
\]
\end{assumption}

\begin{assumption}[Two-view independence for noise (restated Assumption~\ref{ass:tda-indep})]
For a fixed query position $i$ and any noise key $j\in\gN(i)$, the similarities
$\vs_{ij}^{(1)}$ and $\vs_{ij}^{(2)}$ are independent.
\end{assumption}

\subsection{Proof of Theorem \ref{thm:tra-survivors}}
\label[appendix]{ssec:theorem43_proof}
\begin{theorem}[TRA keeps $O(1)$ spurious survivors per row (restated \Cref{thm:tra-survivors})]
\label{thm:tra-survivors-restated}
Assume Assumption~\ref{ass:tra-subg}. Define the number of \emph{spurious survivors} in row $i$ as
\[
S_i \;:=\; \sum_{j\in\gS(i)} \1\!\left(\vs_{ij}>\tau_i\right),
\]
(where the sum ranges over noise keys; for a worst-case statement, you may interpret all keys in $\gS(i)$ as noise).
If the row-wise threshold is
\[
\tau_i \;:=\; \sigma \sqrt{\frac{2\log\!\big(\frac{i+1}{\kappa}\big)}{d}},
\qquad \kappa>0,
\]
then for all $i\ge 1$, $\mathbb{E}[S_i] \le \kappa$.
More specifically, for $\tau_i=\beta \sqrt{\frac{2\log(i+1)}{d}}$ with $\beta>0$,
\[
\mathbb{E}[S_i] \;\le\; (i+1)^{\,1-\beta^2/\sigma^2}.
\]
\end{theorem}

\begin{proof}
Fix a query position $i$ and a noise key $j\in\gN(i)$.
By the sub-Gaussian tail implied by Assumption~\ref{ass:tra-subg}, for any $x\in\mathbb{R}$,
\[
\mathbb{P}(\vs_{ij}>x)\;\le\;\exp\!\Big(-\frac{d x^2}{2\sigma^2}\Big).
\]
By linearity of expectation (no independence across keys is needed),
\[
\mathbb{E}[S_i]
=
\sum_{j\in\gS(i)} \mathbb{P}(\vs_{ij}>\tau_i)
\;\le\;
i \cdot \exp\!\Big(-\frac{d \tau_i^2}{2\sigma^2}\Big).
\]
For $\tau_i = \sigma \sqrt{\frac{2\log((i+1)/\kappa)}{d}}$, we have
$\frac{d\tau_i^2}{2\sigma^2}=\log\!\big(\frac{i+1}{\kappa}\big)$, hence
\[
\mathbb{E}[S_i]
\le
i\exp\!\Big(-\log\!\Big(\frac{i+1}{\kappa}\Big)\Big)
=
\kappa\cdot \frac{i}{i+1}
\le
\kappa.
\]
For $\tau_i=\beta\sqrt{2\log(i+1)/d}$,
\[
\exp\!\Big(-\frac{d\tau_i^2}{2\sigma^2}\Big)
=
(i+1)^{-\beta^2/\sigma^2},
\]
so $\mathbb{E}[S_i]\le i(i+1)^{-\beta^2/\sigma^2}\le (i+1)^{1-\beta^2/\sigma^2}$.
\end{proof}

\subsection{Proof of Corollary \ref{thm:tra-non-dispersion}}
\label[appendix]{ssec:theorem44_proof}
\begin{theorem}[TRA is Non-Dispersive (restated \Cref{thm:tra-non-dispersion})]
\label{thm:tra-non-dispersion-restated}

Under Assumption~\ref{ass:tra-subg} and Assumption~\ref{ass:bounded-relevant},
the Threshold Rectified Attention (TRA) mechanism is non-dispersive.
\end{theorem}

\begin{proof}
Let $S_i$ be the number of non-zero entries (survivors) in the TRA weight vector $\va_i$.
If $\va_i$ has support size $S_i$, then the effective entropy (defined on the $\ell_1$-normalized absolute weights)
is maximized by the uniform distribution on that support, hence
\[
H(\va_i) \le \log(1+S_i),
\]
where the $+1$ handles the case $S_i=0$.

Write $S_i = R_i + N_i$, where $R_i$ is the number of non-noise (relevant) survivors and $N_i$ is the number of noise survivors.
By Assumption~\ref{ass:bounded-relevant}, $R_i \le r$ for all $i$, and by Theorem~\ref{thm:tra-survivors}, if $\beta \ge \sigma$ then $\mathbb{E}[N_i] \le \kappa$.
Applying Jensen's inequality (since $x \mapsto \log x$ is concave),
\[
\mathbb{E}[H(\va_i)]
\le \mathbb{E}[\log(1+S_i)]
\le \log\!\big(1+\mathbb{E}[S_i]\big)
\le \log\!\big(1+r+\kappa\big).
\]
We now take the limit of the dispersion ratio as context length $i \to \infty$:
\[
\lim_{i \to \infty} \frac{\mathbb{E}[H(\va_i)]}{\log i}
\le
\lim_{i \to \infty} \frac{\log(1+r+\kappa)}{\log i}
= 0.
\]
Thus, TRA is non-dispersive.
\end{proof}

\subsection{Proof of Theorem \ref{thm:tda-consensus}}
\label[appendix]{ssec:theorem45_proof}
\begin{theorem}[Consensus spurious survivors vanish in TDA (restated \Cref{thm:tda-consensus})]
\label{thm:tda-consensus-restated}
Under Assumptions \ref{ass:tra-subg} and \ref{ass:tda-indep} for both views (with the same $\sigma$).
Set $\beta=\sigma$, and
\[
\tau_i := \beta \sqrt{\frac{2\log\!\big(\frac{i+1}{\kappa}\big)}{d}}, \qquad \kappa>0,
\]
and define the number of \emph{consensus} spurious survivors
\[
C_i := \sum_{j\in\gN(i)} \1\!\left(\vs_{ij}^{(1)}>\tau_i\right)\,\1\!\left(\vs_{ij}^{(2)}>\tau_i\right).
\]
Then for all $i\ge 1$, $\mathbb{E}[C_i]\le \kappa^2/(i+1)$, hence $\mathbb{E}[C_i]\to 0$ as $i\to\infty$.
\end{theorem}

\begin{proof}
Fix a row $i$ and a noise key $j\in\gN(i)$.
By Assumption~\ref{ass:tra-subg}, each view is mean-zero and $\sigma^2/d$-sub-Gaussian, hence
\[
\mathbb{P}\!\left(\vs_{ij}^{(t)}>\tau_i\right)
\le
\exp\!\Big(-\frac{d\tau_i^2}{2\sigma^2}\Big),
\qquad t\in\{1,2\}.
\]
By Assumption~\ref{ass:tda-indep}, for a noise key $j$ we have independence across views, so
\[
\mathbb{P}\!\left(\vs_{ij}^{(1)}>\tau_i,\; \vs_{ij}^{(2)}>\tau_i\right)
=
\mathbb{P}\!\left(\vs_{ij}^{(1)}>\tau_i\right)\mathbb{P}\!\left(\vs_{ij}^{(2)}>\tau_i\right)
\le
\exp\!\Big(-\frac{d\tau_i^2}{\sigma^2}\Big).
\]
Taking expectation and using linearity,
\[
\mathbb{E}[C_i]
=
\sum_{j\in\gN(i)}
\mathbb{P}\!\left(\vs_{ij}^{(1)}>\tau_i,\; \vs_{ij}^{(2)}>\tau_i\right)
\le
|\gN(i)|\exp\!\Big(-\frac{d\tau_i^2}{\sigma^2}\Big)
\le
(i+1)\exp\!\Big(-\frac{d\tau_i^2}{\sigma^2}\Big).
\]
With $\tau_i=\sigma\sqrt{2\log((i+1)/\kappa)/d}$, we have
$\frac{d\tau_i^2}{\sigma^2}=2\log((i+1)/\kappa)$, hence
\[
\mathbb{E}[C_i]
\le
(i+1)\exp\!\Big(-2\log\!\Big(\frac{i+1}{\kappa}\Big)\Big)
=
(i+1)\Big(\frac{\kappa}{i+1}\Big)^2
=
\frac{\kappa^2}{i+1}.
\]
\end{proof}

\subsection{Proof of Theorem \ref{cor:tda-non-dispersion}}
\label[appendix]{ssec:corollary47_proof}
\begin{corollary}[TDA is Non-Dispersive (restated Corollary~\ref{cor:tda-non-dispersion})]
\label{cor:tda-non-dispersion-restated}
Under Assumption~\ref{ass:tra-subg}, Assumption~\ref{ass:bounded-relevant}, and Assumption~\ref{ass:tda-indep} for both views (with the same $\sigma$),
Threshold Differential Attention (TDA) is non-dispersive.
\end{corollary}

\begin{proof}
TDA computes $\Delta \va = \va^{(1)} - \lambda \va^{(2)}$.
If $\Delta \va_{ij}\neq 0$, then at least one of $\va^{(1)}_{ij},\va^{(2)}_{ij}$ is non-zero, hence the support size satisfies
\[
S_i^{\Delta} \le S_i^{(1)} + S_i^{(2)}.
\]
Write $S_i^{(t)} = R_i^{(t)} + N_i^{(t)}$, where $R_i^{(t)}$ is the number of non-noise (relevant) survivors and $N_i^{(t)}$ is the number of noise survivors in view $t$.
By Assumption~\ref{ass:bounded-relevant}, $R_i^{(t)}\le r$ for all $i$ and $t\in\{1,2\}$.
By Theorem~\ref{thm:tra-survivors} (applied to each view) with $\beta\ge\sigma$, we have $\mathbb{E}[N_i^{(t)}]\le \kappa$.
Therefore,
\[
\mathbb{E}[S_i^{\Delta}]
\le \mathbb{E}[S_i^{(1)}] + \mathbb{E}[S_i^{(2)}]
\le 2(r+\kappa).
\]
Using the same entropy definition (on the $\ell_1$-normalized absolute weights), we have
\[
H(\Delta \va_i) \le \log(1+S_i^{\Delta}),
\]
so by Jensen's inequality,
\[
\mathbb{E}[H(\Delta \va_i)]
\le \mathbb{E}[\log(1+S_i^{\Delta})]
\le \log\!\big(1+\mathbb{E}[S_i^{\Delta}]\big)
\le \log\!\big(1+2(r+\kappa)\big).
\]
Finally,
\[
\lim_{i \to \infty} \frac{\mathbb{E}[H(\Delta \va_i)]}{\log i}
\le \lim_{i \to \infty} \frac{\log(1+2(r+\kappa))}{\log i}
= 0.
\]
\end{proof}

% ------------------------------------------------------------
% Algorithm A.1: Forward
% ------------------------------------------------------------
\begin{tbox}{Threshold Rectified Attention (Forward)}
\label{alg:threla-forward}
\begin{algorithmic}[1]
\Require $Q,K,V \in \mathbb{R}^{B \times H \times T \times D}$; threshold scale $\beta$; $\kappa$;  power $p$; block sizes $B_M,B_N$
\Ensure $O \in \mathbb{R}^{B \times H \times T \times D}$

\Function{ThresholdReLAForward}{$Q,K,V,\beta,p,B_M,B_N$}
  \For{$b \gets 1$ to $B$}
    \For{$h \gets 1$ to $H$}
      \For{$m \gets 0$ to $T-1$ step $B_M$} \Comment{query block start}
        \State $Q_m \gets Q[b,h,\,m:m{+}B_M,\,:]$ \Comment{$B_M \times D$}
        \State Initialize $A \gets \mathbf{0} \in \mathbb{R}^{B_M \times D}$ \Comment{FP32 accumulator}
        \State Compute per-row thresholds:
        \For{$i \gets 0$ to $B_M-1$}
          \State $\tau_m[i] \gets \beta \sqrt{\frac{2\log((m+i+1)/\kappa)}{D}}$
        \EndFor

        \For{$n \gets 0$ to $T-1$ step $B_N$} \Comment{key/value block start}
          \State $K_n \gets K[b,h,\,n:n{+}B_N,\,:]$ \Comment{$B_N \times D$}
          \State $V_n \gets V[b,h,\,n:n{+}B_N,\,:]$ \Comment{$B_N \times D$}
          \State $S \gets Q_m K_n^\top$ \Comment{$B_M \times B_N$}
          \State Apply causal mask (equivalent to the Triton mask):
          \State \hspace{1.2em} $S_{ij} \gets 0 \;\;\textbf{if}\;\; (n+j) > (m+i)$

          \State $X \gets S - \tau_m[:,\mathrm{None}]$ \Comment{broadcast $\tau$ over columns}
          \State $R \gets \max(X,0)$ \Comment{ReLU}
          \State $W \gets R^{p}$ \Comment{elementwise power}

          \State $A \gets A + W V_n$ \Comment{$(B_M\times B_N)(B_N\times D)\rightarrow (B_M\times D)$}
        \EndFor

        \State $O[b,h,\,m:m{+}B_M,\,:] \gets A$
      \EndFor
    \EndFor
  \EndFor
  \State \Return $O$
\EndFunction
\end{algorithmic}
\end{tbox}

\begin{tbox}{Threshold Rectified Attention (Backward)}
\label{alg:threla-backward}
\begin{algorithmic}[1]
\Require Saved $Q,K,V \in \mathbb{R}^{B\times H\times T\times D}$; upstream gradient $dO$; threshold scale $\beta$; $\kappa$; power $p$; block sizes $B_M,B_N$
\Ensure Gradients $dQ,dK,dV$ with same shapes as $Q,K,V$

\Function{ThresholdReLABackward}{$Q,K,V,dO,\beta,p,B_M,B_N$}
  \State Initialize $dQ \gets \mathbf{0}$, $dK \gets \mathbf{0}$, $dV \gets \mathbf{0}$ \Comment{accumulate in FP32}

  \For{$b \gets 1$ to $B$}
    \For{$h \gets 1$ to $H$}

      \For{$m \gets 0$ to $T-1$ step $B_M$} \Comment{query block}
        \State $Q_m \gets Q[b,h,\,m:m{+}B_M,\,:]$ \Comment{$B_M \times D$}
        \State $dO_m \gets dO[b,h,\,m:m{+}B_M,\,:]$ \Comment{$B_M \times D$}
        \State Compute thresholds for this query block:
        \For{$i \gets 0$ to $B_M-1$}
          \State $\tau_m[i] \gets \beta \sqrt{\frac{2\log((m+i+1)/\kappa)}{D}}$
        \EndFor

        \For{$n \gets 0$ to $T-1$ step $B_N$} \Comment{key/value block}
          \State $K_n \gets K[b,h,\,n:n{+}B_N,\,:]$ \Comment{$B_N \times D$}
          \State $V_n \gets V[b,h,\,n:n{+}B_N,\,:]$ \Comment{$B_N \times D$}

          \State Recompute scores: $S \gets Q_m K_n^\top$ \Comment{$B_M \times B_N$}
          \State Apply causal mask: $S_{ij} \gets 0$ if $(n+j)>(m+i)$

          \State $X \gets S - \tau_m[:,\mathrm{None}]$
          \State $R \gets \max(X,0)$
          \State $W \gets R^{p}$

          \State Compute elementwise derivative:
          \State \hspace{1.2em} $G \gets p \cdot R^{p-1} \cdot \mathbb{I}[X>0]$ \Comment{$B_M \times B_N$}

          \State Gradient wrt weights (streaming):
          \State \hspace{1.2em} $dW \gets dO_m V_n^\top$ \Comment{$(B_M\times D)(D\times B_N)$}
          \State \hspace{1.2em} $dS \gets dW \odot G$

          \State Accumulate $dQ$ for this block:
          \State \hspace{1.2em} $dQ[b,h,\,m:m{+}B_M,\,:] \mathrel{+}= dS K_n$ \Comment{$(B_M\times B_N)(B_N\times D)$}

          \State Accumulate $dV$ for this block:
          \State \hspace{1.2em} $dV[b,h,\,n:n{+}B_N,\,:] \mathrel{+}= W^\top dO_m$ \Comment{$(B_N\times B_M)(B_M\times D)$}

          \State Accumulate $dK$ for this block:
          \State \hspace{1.2em} $dK[b,h,\,n:n{+}B_N,\,:] \mathrel{+}= dS^\top Q_m$ \Comment{$(B_N\times B_M)(B_M\times D)$}
        \EndFor
      \EndFor

    \EndFor
  \EndFor

  \State \Return $dQ,dK,dV$
\EndFunction
\end{algorithmic}
\end{tbox}

\begin{tbox}{Threshold Differential Attention (TDA)}
\label{alg:differential-threla}
\begin{algorithmic}[1]
\Require $q_1,q_2,k_1,k_2,v \in \mathbb{R}^{B\times H\times T\times D}$; $\beta$; $\kappa$; $\lambda$; power $p$; \texttt{normalize} flag
\Ensure $O \in \mathbb{R}^{B\times H\times T\times D}$

\Function{DifferentialThresholdReLA}{$q_1,q_2,k_1,k_2,v,\beta,\lambda,p,\texttt{normalize}$}
  \If{\texttt{normalize}}
    \State $q_1 \gets \mathrm{L2Normalize}(q_1)$; \;\; $k_1 \gets \mathrm{L2Normalize}(k_1)$
    \State $q_2 \gets \mathrm{L2Normalize}(q_2)$; \;\; $k_2 \gets \mathrm{L2Normalize}(k_2)$
  \EndIf
  \State $O_1 \gets$ \Call{ThresholdReLAForward}{$q_1,k_1,v,\beta,p,B_M,B_N$}
  \State $O_2 \gets$ \Call{ThresholdReLAForward}{$q_2,k_2,v,\beta,p,B_M,B_N$}
  \State $\lambda_c \gets \min(1,\max(0,\lambda))$ \Comment{clamp to $[0,1]$}
  \State $O \gets O_1 - \lambda_c \cdot O_2$
  \State \Return $O$
\EndFunction
\end{algorithmic}
\end{tbox}

\twocolumn

\section{Empirical Validation of Assumptions}
\label[appendix]{app:theory-diagnostics}

To assess whether the assumptions underlying our analysis are informative in trained models, we report three diagnostics on the trained TDA checkpoint: raw-logit tail behavior, relevant-token survivor scaling, and cross-view correlation. These diagnostics are not intended to verify the assumptions exactly; rather, they test whether the empirical behavior is broadly consistent with the theoretical picture used to motivate the length-dependent threshold and the differential construction.

\subsection{Tail Behavior of Raw Logits}
\label{app:tail-diagnostics}

Our thresholding mechanism is motivated by an extreme-value theory under approximately sub-Gaussian noise (Assumption~\ref{ass:tra-subg}). To examine this, we analyze the raw logits before thresholding:
\[
\vs_{ij} = \vq_i^\top \vk_j / \sqrt{d}.
\]

\Cref{tab:appendix-tail-prob} reports empirical tail probabilities \(P(|\vs| > k\sigma)\) for representative layers, where \(\sigma\) denotes the per-layer standard deviation. While the trained logits are not exactly Gaussian, the empirical tail probabilities remain small and broadly consistent with quadratic-type decay rather than polynomial heavy-tailed behavior. We therefore interpret the sub-Gaussian assumption as an analytically useful approximation rather than an exact empirical model.

\begin{table}[t]
    \centering
    \small
    \begin{tabular}{@{}lccc@{}}
        \toprule
        \textbf{Layer} & \(\sP(|\vs|>2\sigma)\) & \(\sP(|\vs|>3\sigma)\) & \(\sP(|\vs|>4\sigma)\) \\
        \midrule
        0  & 0.0903 & 0.0423 & 0.0143 \\
        2  & 0.1145 & 0.0243 & 0.0062 \\
        4  & 0.0459 & 0.0193 & 0.0153 \\
        6  & 0.0606 & 0.0444 & 0.0181 \\
        8  & 0.0752 & 0.0142 & 0.0017 \\
        11 & 0.1574 & 0.0115 & 0.0003 \\
        \midrule
        $\gN(0,1)$ & 0.0455 & 0.0027 & 6.33e-5 \\
        \bottomrule
    \end{tabular}
    \caption{Empirical tail probabilities \(P(|\vs| > k\sigma)\) for raw logits \(\vs_{ij}=\vq_i^\top \vk_j/\sqrt{d}\), where \(\sigma\) is the per-layer standard deviation. }
    \label{tab:appendix-tail-prob}
\end{table}

\subsection{Relevant-Token Survivor Scaling}
\label{app:survivor-diagnostics}

Our non-dispersion analysis assumes that the number of relevant tokens surviving the adaptive threshold remains bounded (Assumption~\ref{ass:bounded-relevant}). To examine this empirically, we compute the number of surviving keys per query while restricting attention to non-padding relevant tokens only, and evaluate this statistic for context lengths \(n \in \{256,512,1024,2048\}\).

\Cref{fig:appendix-survivor-scaling} shows the mean relevant-token survivor count as context length increases for representative active layers. Across these layers, survivor counts remain stable or grow only mildly relative to a linear \(O(n)\) reference. For example, in Layer 4, the mean relevant-token survivor count increases from \(5.105\) at \(n=256\) to \(6.949\) at \(n=2048\), corresponding to only a \(1.36\times\) increase as the context length grows by \(8\times\). Similar behavior is observed in several other active layers. This supports the bounded-relevant-survivor assumption used in the non-dispersion analysis.

\begin{figure}[t]
    \centering
    \includegraphics[width=\linewidth]{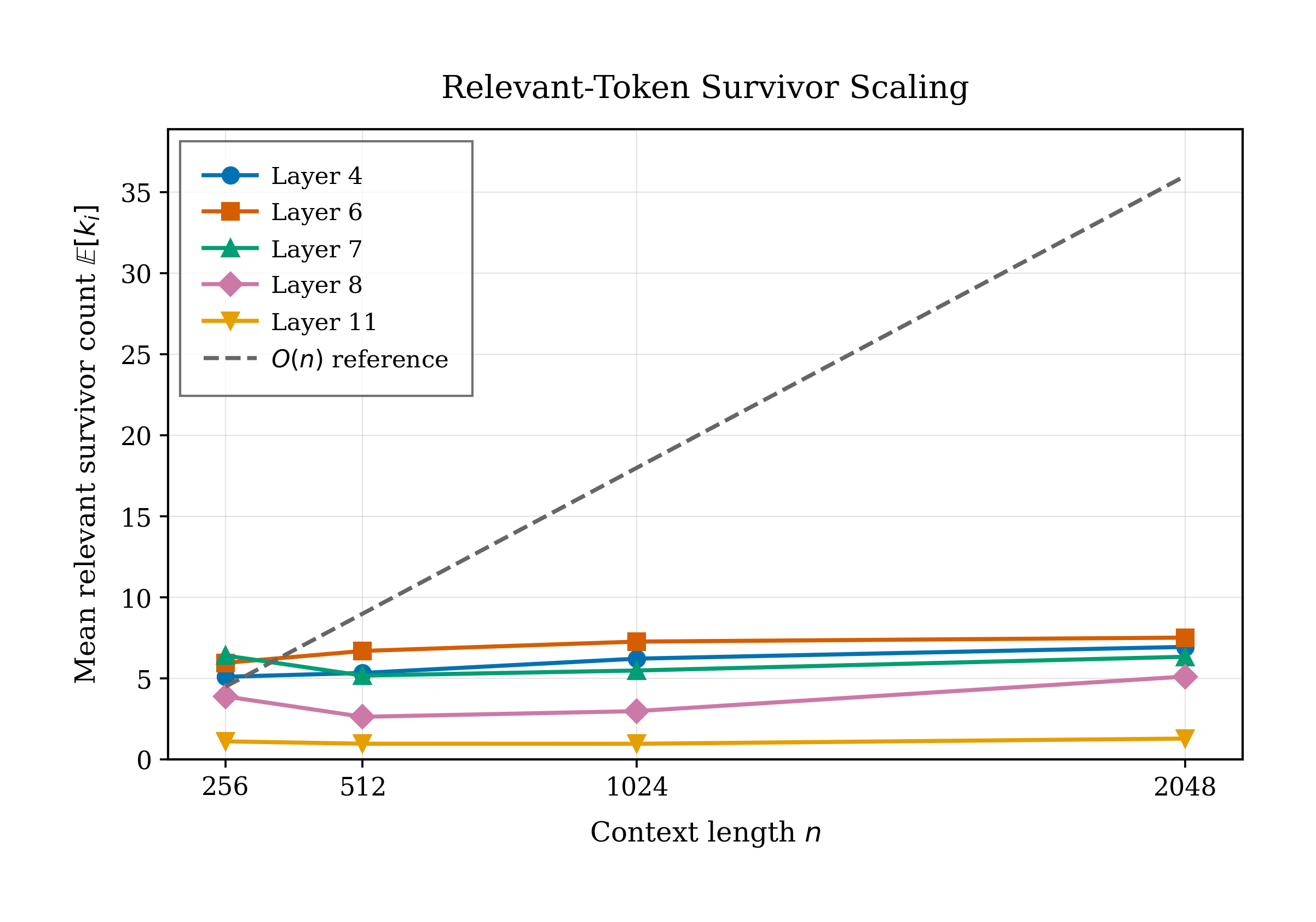}
    \caption{Relevant-token survivor scaling across context length for representative active layers.}
    \label{fig:appendix-survivor-scaling}
\end{figure}

\subsection{Cross-View Correlation}
\label{app:cross-view-diagnostics}

A key concern for differential constructions is whether the two views collapse into a redundant branch after training, which violates the independence of noise views assumption (Assumption~\ref{ass:tda-indep}). To assess this, we measure the Pearson correlation between the two thresholded score views and compare trained and untrained checkpoints.

\Cref{fig:appendix-cross-view-corr} compares layer-wise mean cross-view correlation for trained and untrained models. The trained model exhibits a modest increase in correlation relative to the untrained model, but the correlations remain low overall: the mean\footnote{Some trained layers yield undefined correlation due to one view being nearly constant (almost all zeros). We omit those layers from the average.
} over layer-wise values is \(0.1231\) for the trained checkpoint, compared with \(0.0752\) for the untrained checkpoint. This indicates that the two views do not collapse into a redundant copy after training.

\begin{figure}[t]
    \centering
    \includegraphics[width=\linewidth]{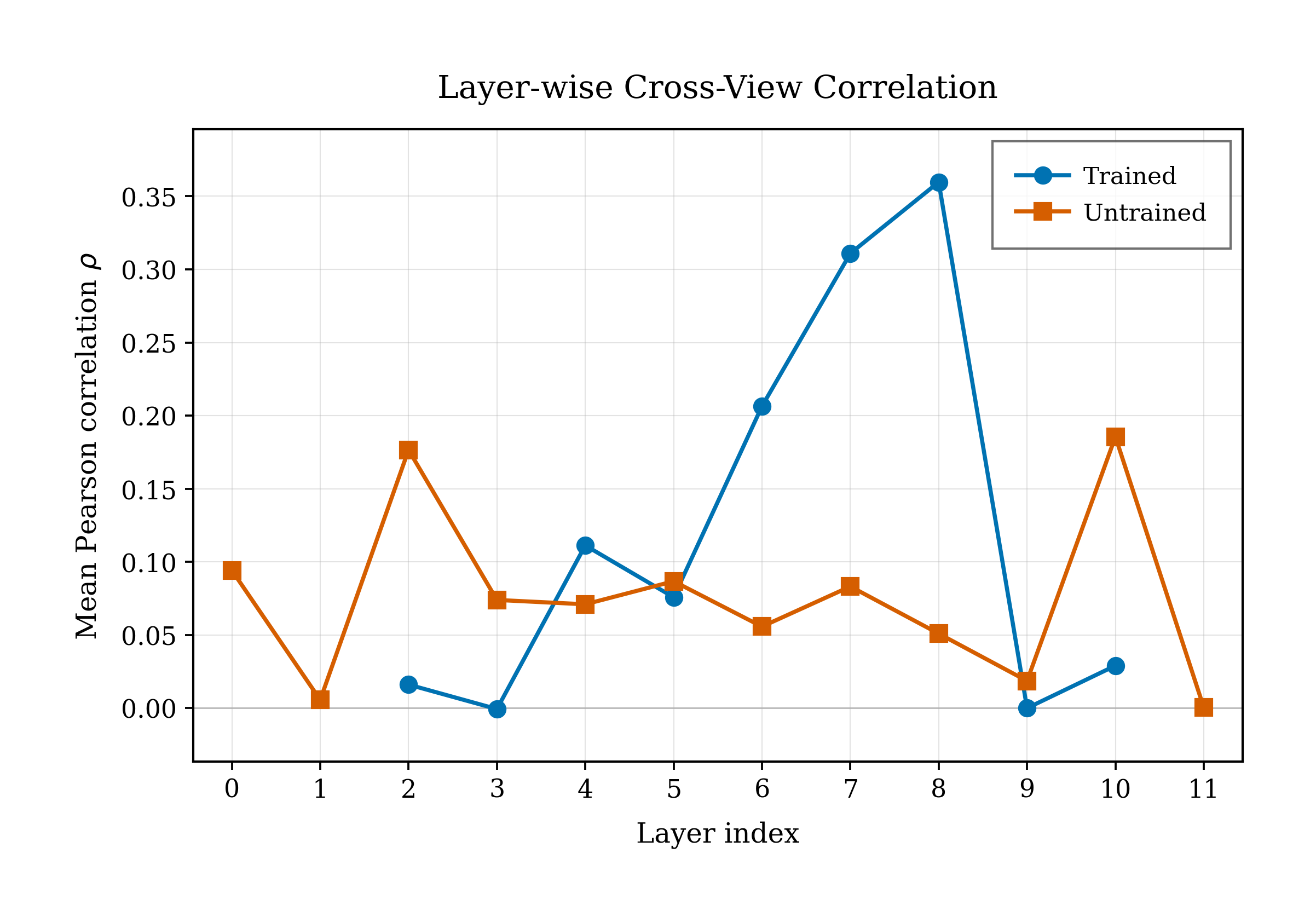}
    \caption{Layer-wise mean cross-view Pearson correlation for trained and untrained models. }
    \label{fig:appendix-cross-view-corr}
\end{figure}

\section{Triton Kernel for TRA and TDA}
\label[appendix]{app:Triton_Kernel}

\begin{figure*}[t]
    \centering
    \begin{subfigure}[b]{0.48\linewidth}
        \centering
        \includegraphics[width=\linewidth]{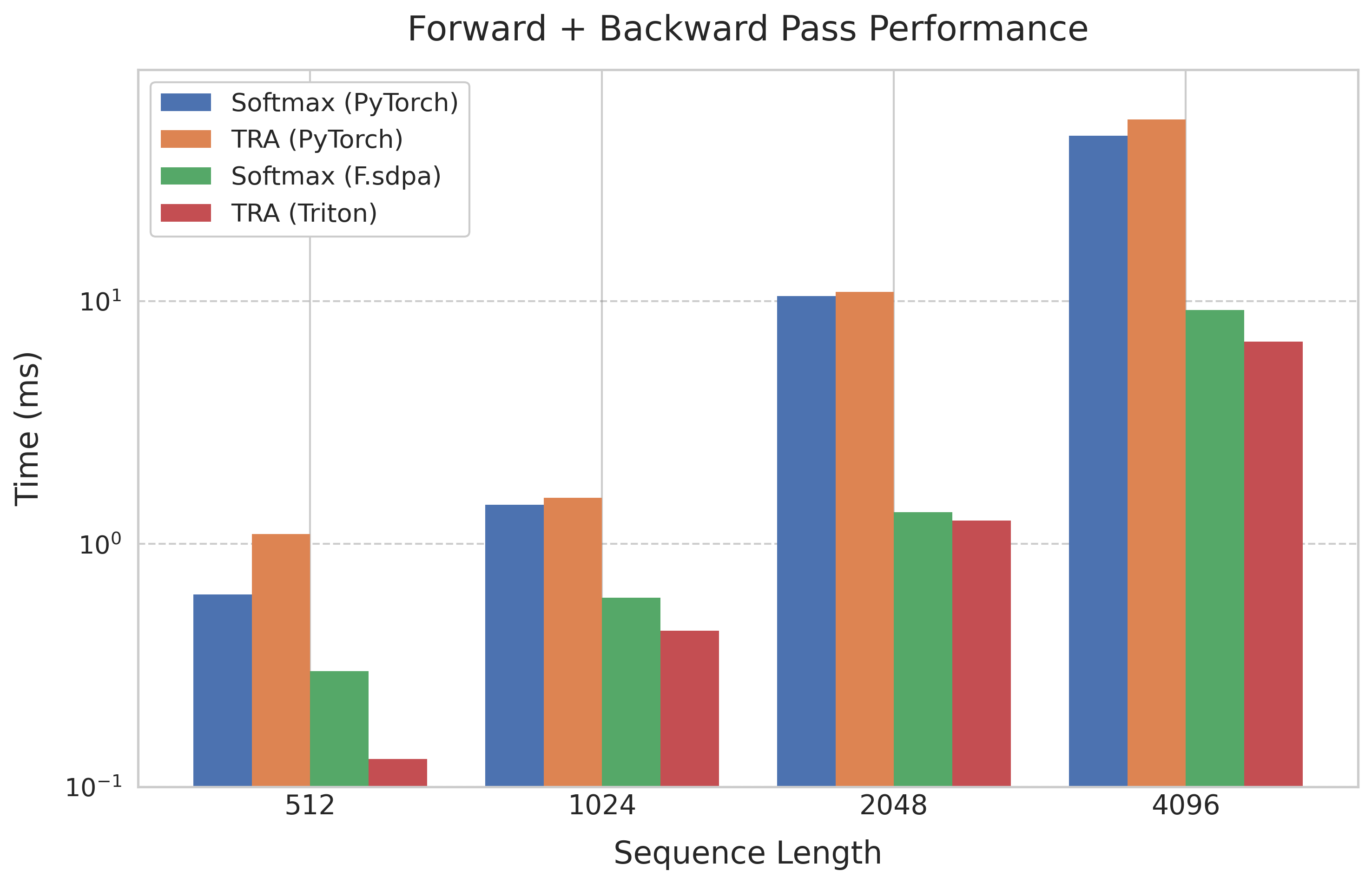}
        \caption{End-to-end forward+backward latency vs.\ sequence length.}
        \label{fig:triton}
    \end{subfigure}
    \hfill
    \begin{subfigure}[b]{0.48\linewidth}
        \centering
        \includegraphics[width=\linewidth]{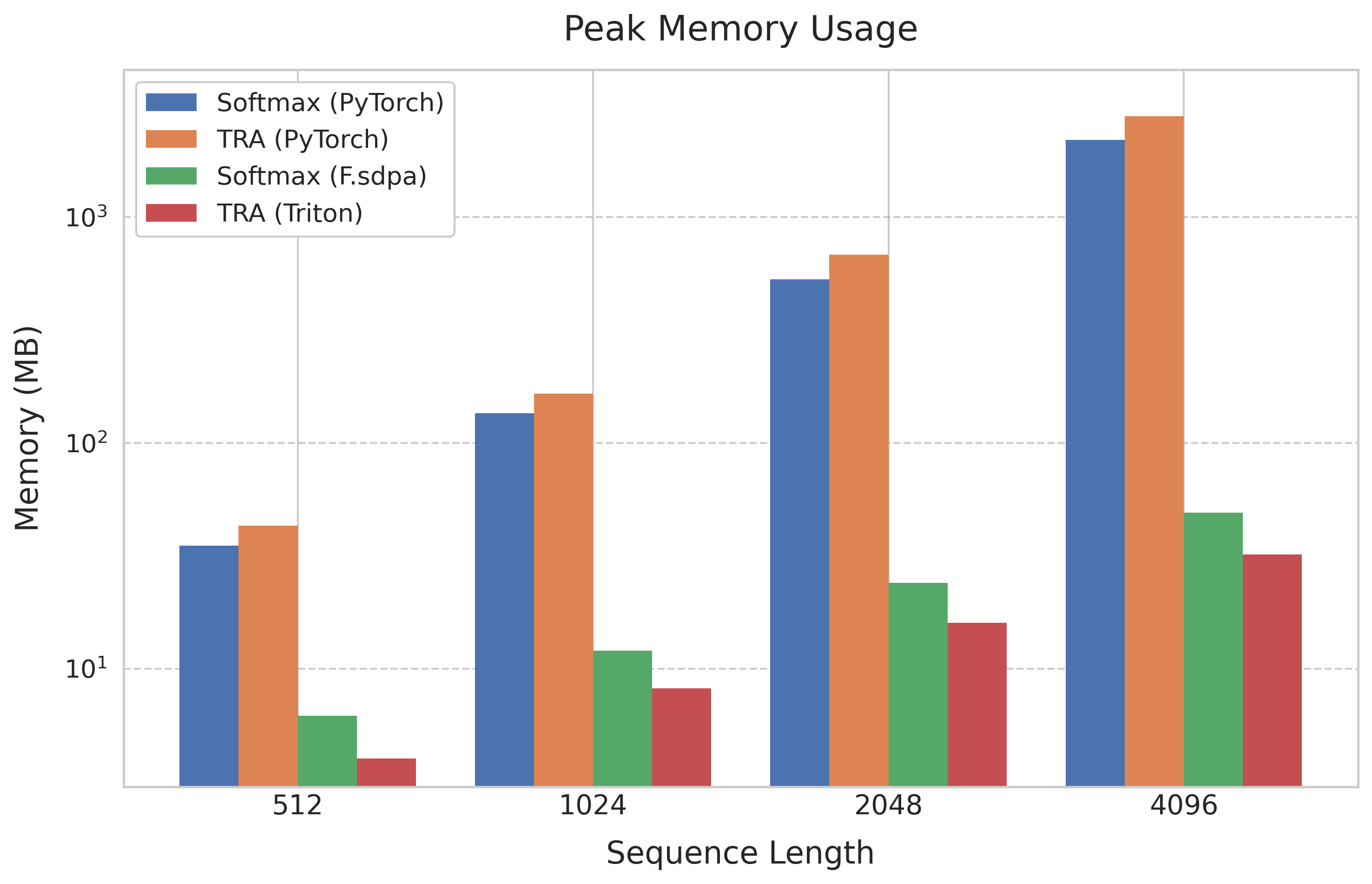}
        \caption{Peak GPU memory during forward+backward.}
        \label{fig:triton-memory}
    \end{subfigure}
    \caption{\textbf{Runtime and memory of fused Triton TRA.} We compare dense Softmax (PyTorch), naive TRA (PyTorch), fused SDPA, and our fused Triton TRA kernel under FP32.}
    \label{fig:triton-combined}
\end{figure*}

\paragraph{Problem setting.}
Given queries $Q$, keys $K$, values $V \in \mathbb{R}^{B \times H \times T \times D}$ (batch $B$, heads $H$, sequence length $T$, head dimension $D$),
Threshold Rectified Attention (TRA) computes:
\begin{align*}
\mO &= \left((\mQ\mK^\top - \tau)_+\right)^{p} \mV, \\
\tau(i) &= \beta \sqrt{\frac{2\log((i+1)/\kappa)}{D}},
\end{align*}    
where $i$ is the (0-indexed) query position, $\beta$ is a scalar threshold scale, $\kappa$ controls the expected number of spurious survivor, and $p$ is the power.
We use causal masking, i.e., query position $i$ can attend only to key positions $j \le i$.

We write $x_+=\max(x,0)$ elementwise. Tiling parameters are $B_M$ (query block size) and $B_N$ (key/value block size).
The implementation is streaming, similar to FlashAttention~\citep{dao2022flashattention,dao2023flashattention2}: it never materializes the full $T\times T$ attention matrix.

\paragraph{Setup.} We additionally benchmark end-to-end \emph{forward+backward} time and
\emph{peak} GPU memory for a single causal attention module at sequence lengths
$T\in\{512,1024,2048,4096\}$ under \textbf{FP32} settings, comparing:
\begin{enumerate}
    \item dense Softmax attention implemented with standard PyTorch operation;
    \item a naive PyTorch implementation of TRA that also materializes scores;
    \item PyTorch fused SDPA (\path{torch.nn.functional.scaled_dot_product_attention}) with Softmax;
    \item our fused Triton implementation of TRA.
\end{enumerate}

Our Triton kernel is streaming in the style of FlashAttention~\citep{dao2022flashattention,dao2023flashattention2}: it computes $\mQ\mK^\top$ in tiles, applies the causal mask and the row-dependent threshold $\tau_i$, and directly accumulates $\sum_j \vw_{ij} \vv_j$ without storing the dense attention matrix, avoiding both the softmax reduction and quadratic activation storage.

\paragraph{Results.} As shown in \Cref{fig:triton-combined}, the naive PyTorch TRA is not competitive: despite using cheaper
nonlinearities than softmax, it still pays the full cost of dense score materialization and unfused elementwise ops.
In contrast, the fused Triton kernel is consistently the fastest option across all tested lengths and improves with
sequence length: at $T{=}4096$ it reduces end-to-end time by several $\times$ compared to dense PyTorch attention and
also outperforms fused SDPA.

Memory gains are even more pronounced: while dense PyTorch implementations scale quadratically and reach multi-GB
peak usage at $T{=}4096$, our Triton kernel matches the FlashAttention-style $O(Td)$ activation footprint, staying in
the tens-of-MB regime. This efficiency is crucial for training TDA at long context, where sparsity is only useful if
the implementation avoids allocating dense attention tensors.

\section{Details in Passkey Retrieval Tests}
\label[appendix]{app:passkey}

Each trial constructs a single prompt consisting of five blocks separated by newline characters:

\begin{enumerate}
    \item \textbf{Task description (instruction).} A fixed instruction string:
\begin{quote}\small
There is an important info hidden inside a lot of irrelevant text. Find it and memorize them. I will quiz you about the important information there.
\end{quote}

    \item \textbf{Garbage prefix.} A prefix of length $n_{\mathrm{garbage\_prefix}}$ characters cut from a long ``garbage'' string:
\begin{quote}\small
The grass is green. The sky is blue. The sun is yellow. Here we go. There and back again.
\end{quote}
This base sentence is repeated many times to form a long buffer, and we take a random-length prefix substring.

    \item \textbf{Passkey line (needle).} A single line containing the passkey, where the key is repeated to reduce ambiguity:
\begin{quote}\small
The pass key is \{$x$\}. Remember it. \{$x$\} is the pass key.
\end{quote}
The passkey $x$ is sampled uniformly as an integer $x \sim \mathrm{Unif}\{1,\dots,50000\}$.

\item \textbf{Garbage suffix.} A suffix of length $n_{\mathrm{garbage\_suffix}}$ characters cut from the same repeated garbage buffer, where
\(
n_{\mathrm{garbage\_prefix}} + n_{\mathrm{garbage\_suffix}} = n_{\mathrm{garbage}}.
\)

    \item \textbf{Final query.} A fixed question that prompts the model to output the passkey:
\begin{quote}\small
What is the pass key? The pass key is
\end{quote}
\end{enumerate}

The final prompt is the newline-joined concatenation of these blocks:
\begin{align}
\texttt{prompt} = \texttt{instr} &\;\Vert\; \texttt{garbage\_prefix} \;\Vert\; \texttt{needle} \nonumber \\
&\;\Vert\; \texttt{garbage\_suffix} \;\Vert\; \texttt{query}, \nonumber
\end{align}

where $\Vert$ denotes concatenation with newline separators.

For each target length $T_{\mathrm{target}}\in\{500,1000,\dots,4000\}$, we construct a prompt by inserting a numeric passkey statement at a random location within a long span of irrelevant ``garbage'' text.
We run 100 trials per length with independently sampled passkeys and insertion positions, use greedy decoding, and count a trial as correct if the generated answer matches the ground-truth passkey (exact string match after stripping whitespace, or integer match).
This follows the standard needle-in-a-haystack evaluation protocol~\citep{mohtashami2023randomaccess,kamradt2023needle}.

\section{Further Experimental Details}

\label[appendix]{app:exp}

\begin{table*}[t]
\centering
\small
\setlength{\tabcolsep}{6pt}
\begin{tabular}{l r r r l}
\toprule
\textbf{Dataset} & \textbf{Train} & \textbf{Validation} & \textbf{Test} & \textbf{Task Type} \\
\midrule
\textbf{HellaSwag}          & 39,905 & 10,042 & 10,003 & Commonsense completion \\
\textbf{ARC-Easy}           & 2,251  & 570    & 2,376  & Multiple-choice science QA \\
\textbf{ARC-Challenge}      & 1,119  & 299    & 1,172  & Multiple-choice science QA \\
\textbf{OpenBookQA}         & 4,957  & 500    & 500    & Multiple-choice QA \\
\textbf{PIQA}               & 16,113 & 1,838  & 3,084  & Physical commonsense QA \\
\textbf{WinoGrande}     & 40,398 & 1,267  & 1,767  & Coreference resolution \\
\bottomrule
\end{tabular}
\caption{Statistics for language modelling benchmarks.}
\label{tab:short_context_stats}
\end{table*}

\begin{table*}[t]
\centering
\small
\setlength{\tabcolsep}{6pt}
\begin{tabular}{l l cccc}
\toprule
\textbf{Dataset} & \textbf{Domain} & \textbf{Train} & \textbf{Validation} & \textbf{Test} & \textbf{Avg.\ Length (Words)} \\
\midrule
\textbf{QMSum}            & Meetings          & 1,257  & 272  & 281  & 9,497 \\
\textbf{SummScreenFD}  & Screenplays       & 3,673  & 338  & 337  & 5,598 \\
\textbf{GovReport}        & Government        & 17,457 & 972  & 973  & 7,886 \\
\textbf{Qasper}           & Scientific papers & 2,567  & 1,726& 1,399& 3,629 \\
\bottomrule
\end{tabular}
\caption{Statistics for SCROLLS benchmarks used in long-context evaluation. Avg.\ length is reported in words.}
\label{tab:scrolls_stats}
\end{table*}

\subsection{Sparsity}
\label{app:sparsity_def}

We quantify sparsity by the fraction of attention entries that are \emph{exactly zero}.
For a given layer $\ell$ and head $h$, let $\mathbf{A}^{\ell,h}\in\mathbb{R}^{T\times T}$ denote the (masked) attention weight matrix produced by the mechanism, where causal masking enforces $\mathbf{A}^{\ell,h}_{ij}=0$ for $j>i$.
We call an entry \emph{inactive} if $\mathbf{A}^{\ell,h}_{ij}=0$ exactly, and define the per-head sparsity for the causal-masked transformer as

\begin{align*}
\mathrm{Sparsity}^{\ell,h}
\;&:=\;
\frac{\sum_{i=1}^{T}\sum_{j=1}^{T} \mathbf{1}\!\left[\mathbf{A}^{\ell,h}_{ij}=0\right]}
{\sum_{i=1}^{T}\sum_{j=1}^{T} \mathbf{1}\!\left[j\le i\right]}.
\label{eq:sparsity_def}
\end{align*}
We then aggregate across heads and layers by averaging:
\begin{equation*}
\mathrm{Sparsity}
\;:=\;
\frac{1}{LH}\sum_{\ell=1}^{L}\sum_{h=1}^{H}\mathrm{Sparsity}^{\ell,h}.
\end{equation*}

\subsection{Dataset Statistics}
We report the statistics for the datasets used in our standard and long-context evaluations in \Cref{tab:short_context_stats} and \Cref{tab:scrolls_long_context}.

For the standard language modeling evaluation, we use the following zero-shot benchmarks. These datasets typically consist of short contexts (questions or partial sentences) suitable for testing core reasoning capabilities.

For long-context evaluation, we select four diverse tasks from the SCROLLS benchmark~\citep{shaham2022scrolls}. These datasets feature input lengths significantly exceeding the training context of standard models, requiring the model to effectively extrapolate or manage long-range dependencies.

\subsection{Training Configuration}

\paragraph{Training} For standard configurations (GPT-2-162M), we use a total batch size of 524,288 tokens per gradient update step, with a mini-batch size of 16 sequences per GPU and a context length of 1024 tokens. Gradient accumulation steps are automatically calculated to achieve the desired total batch size across all GPUs. The dataset is tokenized using the GPT-2 tokenizer (tiktoken encoding) with a vocabulary size of 50,304 tokens.

All models are trained for 5 epochs with 38,146 steps per epoch, resulting in approximately 190,730 total training steps. We evaluate on the validation set every 250 steps, using 20 validation batches per evaluation.

\paragraph{Hyperparameters.} We employ a learning rate schedule with linear warmup followed by cosine decay. The maximum learning rate is set to $1 \times 10^{-3}$, with a minimum learning rate of $1 \times 10^{-4}$. Warmup is performed over 715 steps. We use a weight decay of 0.1 and set the random seed to 1337 for reproducibility. For models with RoPE positional encoding, we use a base frequency $\theta = 10,000$.
When extending models to longer context lengths, we employ NTK-aware scaling for RoPE-based models \citet{peng2024yarn,ntk-aware-scaling}, and further finetune them for 500 additional steps. We report single-run results for all evaluations, as the variance is small.

\subsection{Hardware and Infrastructure}
All experiments were conducted on NVIDIA A100-80GB GPUs. With Triton kernel optimizations enabled for threshold-based attention mechanisms, memory usage per GPU was approximately 35--45GB.

\section{Use of Large Language Model}
\label{app:use_llm}
We used large language models (LLMs) to assist with writing by refining human-written text and to support code development.

\end{document}

%% file: math_commands.tex
\usepackage{amsmath,amsfonts,bm}

\newcommand{\newterm}[1]{{\emph{#1}}}

\def\eqref#1{equation~\ref{#1}}
\def\1{\bm{1}}

\def\vzero{{\bm{0}}}

\def\va{{\bm{a}}}

\def\vk{{\bm{k}}}

\def\vo{{\bm{o}}}

\def\vq{{\bm{q}}}

\def\vs{{\bm{s}}}

\def\vv{{\bm{v}}}
\def\vw{{\bm{w}}}
\def\vx{{\bm{x}}}

\def\mA{{\bm{A}}}

\def\mH{{\bm{H}}}

\def\mK{{\bm{K}}}

\def\mO{{\bm{O}}}

\def\mQ{{\bm{Q}}}

\def\mS{{\bm{S}}}

\def\mV{{\bm{V}}}
\def\mW{{\bm{W}}}
\def\mX{{\bm{X}}}

\DeclareMathAlphabet{\mathsfit}{\encodingdefault}{\sfdefault}{m}{sl}
\SetMathAlphabet{\mathsfit}{bold}{\encodingdefault}{\sfdefault}{bx}{n}

\def\gN{{\mathcal{N}}}

\def\gR{{\mathcal{R}}}
\def\gS{{\mathcal{S}}}

\def\gV{{\mathcal{V}}}

\def\sP{{\mathbb{P}}}

\newcommand{\R}{\mathbb{R}}

\newcommand{\softmax}{\mathrm{softmax}}

%% file: table_results.tex
\begin{table*}[t]
\centering
% \scriptsize
\setlength{\tabcolsep}{2pt}
\resizebox{0.99\textwidth}{!}{
\begin{tabular}{l c cc cc cc cc cc c c}
\toprule
\multirow{2}{*}{\textbf{Method}} &
\multirow{2}{*}{\textbf{Val.\ Loss} $\downarrow$} &
\multicolumn{2}{c}{\textbf{HellaSwag}} &
\multicolumn{2}{c}{\textbf{ARC-Easy}} &
\multicolumn{2}{c}{\textbf{ARC-Challenge}} &
\multicolumn{2}{c}{\textbf{OpenBookQA}} &
\multicolumn{2}{c}{\textbf{PIQA}} &
\multicolumn{1}{c}{\textbf{Winogrande}} &
\multirow{2}{*}{\textbf{Sparsity} $\uparrow$} \\
\cmidrule(lr){3-4}
\cmidrule(lr){5-6}
\cmidrule(lr){7-8}
\cmidrule(lr){9-10}
\cmidrule(lr){11-12}
\cmidrule(lr){13-13}
& & Acc & Acc-Norm & Acc & Acc-Norm & Acc & Acc-Norm & Acc & Acc-Norm & Acc & Acc-Norm & Acc & \\
\midrule

% Softmax and variants
Softmax                & \underline{3.1196} & \textbf{0.345} & 0.409 & \textbf{0.526} & 0.487 & 0.223 & 0.245 & 0.180 & 0.304 & \underline{0.641} & 0.621 & 0.490 & 0\% \\
Gated Softmax          & 3.1489 & 0.330 & 0.382 & 0.474 & 0.436 & 0.194 & 0.224 & 0.162 & 0.284 & 0.620 & 0.586 & 0.500 & 0\% \\
SSMax       & 3.1369 & 0.324 & 0.387 & 0.472 & 0.462 & 0.191 & 0.231 & 0.144 & 0.302 & 0.620 & 0.590 & 0.508 & 0\% \\
\midrule

% Non-softmax (non-thresholded)
Entmax                 & 3.1941 & \underline{0.342} & 0.391 & 0.508 & 0.472 & 0.194 & 0.245 & \underline{0.198} & 0.304 & 0.632 & 0.609 & \underline{0.523} & 43\% \\
LSSA                  & 3.1676 & 0.330 & 0.378 & 0.521 & 0.470 & 0.217 & \textbf{0.259} & 0.192 & 0.314 & 0.652 & 0.621 & \textbf{0.532} & 0\% \\
ReLA                   & 3.1657 & 0.329 & 0.394 & 0.512 & 0.468 & \textbf{0.226} & 0.250 & 0.194 & 0.306 & 0.634 & 0.621 & 0.509 & 94\% \\
\midrule

% Differential
Diff Softmax   & 3.1941 & 0.336 & \textbf{0.423} & 0.509 & \textbf{0.496} & 0.225 & \underline{0.252} & 0.178 & \underline{0.316} & \textbf{0.648} & 0.619 & 0.514 & 0\% \\
Dex                    & 3.1349 & 0.339 & 0.395 & 0.492 & 0.466 & 0.215 & 0.241 & 0.172 & 0.282 & 0.640 & 0.608 & 0.519 & 0\% \\
Diff ReLA      & 3.1294 & 0.331 & 0.391 & 0.514 & 0.472 & 0.220 & 0.248 & 0.192 & 0.298 & 0.636 & \underline{0.623} & 0.494 & \underline{96\%} \\
\midrule
\midrule
% Threshold attention (ours)
TRA ($p{=}2,\beta{=}1$) & 3.1320 & 0.330 & 0.401 & 0.516 & 0.471 & \textbf{0.226} & 0.247 & 0.194 & 0.278 & 0.637 & \textbf{0.626} & 0.496 & 92\% \\
TDA ($p{=}2,\beta{=}1$) & \textbf{3.1190} & 0.337 & \underline{0.415} & \underline{0.524} & \underline{0.488} & 0.220 & 0.239 & \textbf{0.216} & \textbf{0.320} & 0.628 & \textbf{0.626} & 0.489 & \textbf{99\%} \\
\bottomrule
\end{tabular}}
\caption{\textbf{Language modeling results.} We report validation loss and both accuracy (Acc) and length-normalized accuracy (Acc-Norm). Winogrande reports Acc only. \textbf{Sparsity} (see \cref{app:sparsity_def}) is the fraction of attention weights that are exactly $0$, averaged over layers, heads, and instances. We \textbf{bold} the first and \underline{underscore} the second place.}
\label{tab:lmeval_acc_norm_loss_sparsity}
\end{table*}